\documentclass{article}

\usepackage[preprint]{neurips_2026}

\usepackage[utf8]{inputenc}
\usepackage[T1]{fontenc}
\usepackage{hyperref}
\usepackage{url}
\usepackage{booktabs}
\usepackage{amsfonts}
\usepackage{nicefrac}
\usepackage{microtype}
\usepackage{xcolor}
\usepackage{graphicx}
\usepackage{subcaption}
\usepackage{algorithm}
\usepackage{algorithmic}
\usepackage{amsmath}
\usepackage{amssymb}
\usepackage{mathtools}
\usepackage{amsthm}

\newcommand{\cA}{\mathcal{A}}
\newcommand{\cB}{\mathcal{B}}

\newcommand{\cF}{\mathcal{F}}
\newcommand{\cG}{\mathcal{G}}
\newcommand{\cH}{\mathcal{H}}

\newcommand{\cN}{\mathcal{N}}
\newcommand{\cO}{\mathcal{O}}

\newcommand{\cS}{\mathcal{S}}
\newcommand{\cT}{\mathcal{T}}

\newcommand{\cV}{\mathcal{V}}

\newcommand{\cZ}{\mathcal{Z}}

\newcommand{\bfk}{\mathbf{k}}

\newcommand{\bfB}{\mathbf{B}}

\newcommand{\bfG}{\mathbf{G}}

\newcommand{\bfI}{\mathbf{I}}

\newcommand{\bfK}{\mathbf{K}}

\newcommand{\bfW}{\mathbf{W}}

\newcommand{\bfY}{\mathbf{Y}}
\newcommand{\bfZ}{\mathbf{Z}}

\theoremstyle{plain}
\newtheorem{theorem}{Theorem}[section]

\newtheorem{lemma}[theorem]{Lemma}
\newtheorem{corollary}[theorem]{Corollary}
\theoremstyle{definition}
\newtheorem{definition}[theorem]{Definition}
\newtheorem{assumption}[theorem]{Assumption}

\theoremstyle{remark}

\title{Learning Kernel-Based MDPs from Episodic Preferential Feedback}

\author{%
  Nikola Pavlovic\\
  Cornell University\\
  \texttt{np.358@cornell.edu} \\
  \And
  Sattar Vakili\\
  MediaTek Research\\
  \AND
  Qing Zhao \\
  Cornell University\\
}

\begin{document}

\maketitle

\begin{abstract}
Human feedback often arrives as preferences rather than calibrated numeric 
rewards, motivating reinforcement learning from preferential feedback, also 
referred to as reinforcement learning from human feedback (RLHF). We present 
a rigorous theoretical study of preference-only learning in episodic kernel 
MDPs. In each episode, the learner deploys two policies from a common start 
state and receives a single binary label indicating which trajectory is 
preferred, modeled by a Bradley--Terry--Luce link on the difference of 
cumulative (unobserved) rewards. Under kernel-based assumptions on the reward 
and transition functions (one of the most general models amenable to theoretical 
analysis) we develop preference-based value estimation and confidence sets 
tailored to end-of-episode comparisons.We
prove high-probability regret bounds that scale
sublinearly in the number of episodes, implying
that the value of the learned policy converges to
that of the optimal policy.
\end{abstract}

\section{Introduction}

Optimizing sequential decision-making using only preference-based feedback has recently emerged as an important problem \citep{wirth2017survey, christiano2017deep}. In many applications, including large language model alignment \citep{ouyang2022training}, robotics, and recommendation systems, humans find it natural to compare options rather than assign calibrated numeric scores. Preference feedback is often more reliable than numeric ratings, even if less informative. Since human feedback is costly, it becomes essential to develop efficient methods that can learn near-optimal policies while minimizing the number of preference queries required. This motivates the study of reinforcement learning from preferential feedback, commonly referred to as reinforcement learning from human feedback (RLHF).

In the RLHF framework for episodic Markov Decision Processes (MDPs), at each episode, the learner deploys policies and receives preference feedback indicating which trajectory is preferred, rather than observing numeric rewards. This binary feedback is typically modeled as a Bernoulli random variable using the Bradley-Terry-Luce (BTL) model, in which the probability of a preference is determined by applying a link function to the difference in cumulative rewards. Performance is measured in terms of regret, defined as the cumulative loss compared to the optimal policy. The key challenge is learning from minimal one-bit feedback per episode, where the trajectory-level signal provides no direct information about per-step credit assignment.

Several works have studied preference-based RL in tabular MDPs, but these 
algorithms are computationally intractable even for small state spaces. 
\citet{wu2023making} provide the first computationally efficient algorithm 
for preference-based RL in linear MDPs, achieving near-optimal regret via 
randomization to handle trajectory-level feedback. However, the linear MDP 
assumption is restrictive. For example, \citet{lee2023demystifying} showed that the required 
feature dimension can scale with the size of the state space. Kernel-based 
models offer a natural generalization that captures nonlinear structure while 
remaining amenable to theory, and connect to neural networks via the Neural 
Tangent Kernel~\citep[NTK,][]{jacot2018neural}. Kernel RL under standard 
reward feedback has been studied extensively
\citep{yang2020function,chowdhury2019online,maran2024no,maran2024local}, and sublinear regret has only recently been achieved \citep{maran2024local}.

 \subsection{Main Results}\label{Sec: Main_Results}

In this work, we develop the first algorithm for preference-based RL in kernel MDPs that achieves provable sublinear regret for all kernels in the Matérn family. Our algorithm attains cumulative regret of order $\widetilde{\cO}\left(K^{1-0.5(\beta_p-1)^2/\beta_p(\beta_p+1)}\right)$, where $\beta_p>1$ is the eigen-decay parameter defined in Assumption~\ref{ass:eigen_assumption}. In addition, the algorithm can be implemented in time polynomial in the number of episodes $K$, thereby providing the first computationally efficient no-regret algorithm for preference-based feedback in kernel MDPs.

In the kernel RL literature~\citep{vakili2021optimal, whitehouse2023sublinear, maran2024local}, achieving sublinear regret in the number of episodes has been a central goal, as it implies convergence to the optimal policy as the number of episodes grows. For Matérn kernels, sublinear regret was only very recently established for MDPs with standard scalar-valued reward feedback in \citet{maran2024local}.\footnote{
Although a sharper upper bound is claimed in \cite{Vakili23-KerRL}, the derivation appears to rely on a subtle issue in the covering argument, preventing the stated regret bound from following. This issue is discussed in detail in
\citet[App.~F]{maran2024local}; see also Sec.~\ref{sec:discuss_appendix} for a broader discussion.
}
The authors develop a no-regret algorithm for scalar-feedback MDPs over a broad class of smooth environments, with kernel MDPs arising as a special case. However, the currently known implementations are not computationally efficient
\citep[App.~D]{maran2024no}; in particular, it remains unclear whether they can be implemented in time polynomial in the number of episodes. We discuss computational complexity in kernel MDPs further in Sec.~\ref{subsec:comp_complexity}.

Preference feedback is substantially more challenging than action-level scalar reward feedback. Unlike scalar feedback, which provides continuous signals containing both magnitude and distributional information, preference feedback compresses the information from an entire episode into a single binary comparison. As a result, both amplitude information (reward magnitude) and fine-grained temporal information (per-step reward assignment) are lost. Nevertheless, we show that \emph{near-optimal policies can be learned using only $K$ bits of feedback}, corresponding to one binary preference label per episode across $K$ episodes.

Achieving sublinear regret under preference-based feedback requires overcoming three interrelated technical challenges: (i) controlling the \emph{log-covering number} of the value function class, which depends on the RKHS norm of the estimators; (ii) designing exploration noise that induces sufficient trajectory-level exploration while remaining amenable to efficient covering arguments; and (iii) carefully tuning regularization and covering parameters to balance bias, variance, and statistical complexity simultaneously.

\subsection{Technical Challenges and Contributions}\label{sec:tech_contributions}
We provide an overview of the three main technical challenges and corresponding contributions here, with full technical details deferred to Sec.~\ref{sec:per_gaurantees_tech_challenge}.
\paragraph{(i) Controlling the log-covering number via estimator norm bound.}
In kernel MDPs, confidence bounds depend on the log-covering number of the value-function class, which in turn scales with the RKHS norms of the estimators. This is a fundamental difference from linear MDPs, where the log-covering numbers are only logarithmic in the episode count \citep{jin2020provably}.  Under trajectory-level preference feedback, localization arguments used in action-level kernel RL \citep{maran2024no, janz2020bandit}, no longer apply, and naive estimators can exhibit rapidly growing RKHS norms, leading to overly conservative confidence sets. Our approach controls this complexity by learning the reward component from preferences using regularized kernel logistic regression, with the regularization explicitly tuned to keep the estimator norm (and thus the covering number of the induced function class) small.

\paragraph{(ii) Designing trajectory-level exploration with efficiently coverable sample paths.}
Preference feedback requires exploration at the trajectory level, but deterministic bonuses cannot be applied because the realized trajectory is unknown at decision time. We therefore use randomized exploration through a Gaussian process perturbation that can be injected online. Since GP sample paths do not belong to the RKHS induced by their covariance kernel
almost surely \citep{kanagawa1807gaussian}, we cannot cover the randomized estimator by placing it in a ball of the original RKHS. Instead, we use a high-probability Gaussian covering set: although a GP sample is not
contained in any fixed finite-radius RKHS ball, it lies with high probability within\(\epsilon\), in \(\ell_\infty\), of a suitably enlarged RKHS ball
 \citep{van2008rates}. This allows us to derive covering bounds for the posterior GP perturbation directly.

\paragraph{(iii) Tuning regularization and covering parameters to balance 
bias, variance, and complexity.}
Regularization plays a dual role in our work: it controls statistical accuracy (bias and 
variance) and also complexity via covering numbers. Unlike linear value iteration 
\cite{jin2020provably}, where log covering numbers are only polylogarithmic 
in the mesh $\epsilon$, kernel MDPs yield polynomial log covering bounds 
(see~\eqref{eq:main_text_cover}). This makes tuning $\epsilon$ non-trivial: 
too small inflates confidence widths, too large inflates misspecification 
error. We tune both the regularization and covering mesh $\epsilon$ jointly 
to keep confidence bounds meaningful under preference feedback.


Together, these ingredients yield the first preference-based kernel MDP algorithm with provable sublinear regret.

\subsection{Related Work}

In recent years, preference-based reinforcement learning has gained significant attention in both bandit and MDP settings. In the bandit setting, that corresponds to a degenerate case with a single state, action-level preference (dueling bandits) has been extensively studied across various families of reward functions and preference assumptions \citep{yue2009interactively, ailon2014reducing, saha2021adversarial, bengs2022stochastic, wu2023borda, verma2024neural, kayal2025bayesian}\footnote{The body of literature on dueling bandits is vast, see  \citet{bengs2021preference} for more references}.\\

Preference-based MDPs have also received significant attention; however, theoretical guarantees in this setting often require optimization over a large policy space, making the approach computationally infeasible. \citet{pacchiano2021dueling} consider linear preference-based MDPs with trajectory-level preference and provide sub-linear regret guarantees. However, in a linear setting, the transition model is assumed to be known. Furthermore, the algorithm requires optimization across a large policy class, making it unclear how the approach could be implemented. \citet{chen2022human} further generalizes \citet{pacchiano2021dueling}, allowing the preference function to be an unknown element of a large function space; however, as in \citet{pacchiano2021dueling}, it also requires optimization over the entire policy space. The works of \citet{zhan2023provable, zhan2023query} suffer from similar scaling issues due to their policy-search approach. To the best of our knowledge, the first algorithm for preference-based feedback in linear MDPs that can be implemented efficiently is due to \citet{wu2023making}. To bypass the total policy search, \citet{wu2023making} incites trajectory-level exploration by injecting adequately sampled noise. The idea of randomized trajectory-level exploration has also appeared in works studying aggregate feedback \citep{cassel2024near, shani2020optimistic, efroni2021reinforcement}. Preference feedback is strictly harder than aggregate feedback: while aggregate feedback provides a continuous signal in $[0, H]$, preference feedback compresses this to a single bit, losing both the magnitude of reward differences and fine-grained distributional information\\
Also relevant to our work are papers on kernel MDPs with standard, immediate, state-action feedback. The setting in which the reward and transition operators are general members of an RKHS was first considered by \citet{yang2020function}. In the paper, the authors provide regret guarantees for a wide class of kernels but do not ensure regret sublinearity kernel  in the Mat\'ern family. Under the kernel MDP assumption, the authors of \citet{vakili2024kernel} provide sublinear regret guarantees for the more challenging average-reward setting but with a constraining "optimistic closure" assumption on the value function which removes the need for value iteration. Several other works have considered kernel-based RL under additional structural assumptions \citep{yang2020reinforcement, chowdhury2019online} but do not obtain sublinear regret for Mat\'ern kernels of general smoothness.\\
Recently, the study of smooth state-action feedback MDPs has emerged \citep{maran2024no,maran2024local}. Instead of imposing smoothness assumptions directly on the MDPs reward function or transition operator, as in our work, smoothness is imposed on the Bellman operator. The Mat\'ern kernel assumption implies the smoothness of the Bellman operator and is studied as a special case of a more general framework in \cite{maran2024no,maran2024local}. To the best of our knowledge, the algorithm in \cite{maran2024local} is the \textit{only} one that achieves sublinear regret for all Mat\'ern kernels.

\section{Problem Formulation}\label{sec:problem}

We consider an episodic Markov decision process $\mathrm{MDP}(\mathcal S,\mathcal A,H,\{P_h\}_{h=1}^H,\{r_h\}_{h=1}^H)$ with measurable state space $\mathcal S$,  action set $\mathcal A$ and horizon $H\in\mathbb N$.  For convenience of analysis, we assume that the state action product set is contained in a unit hypercube over $\mathbb{R}^{d}$. At step $h\in[H]$, taking action $a_h\in\mathcal A$ at state $x_h\in\mathcal S$ yields (deterministic) reward $r_h(x_h,a_h)\in[0,1]$ and transitions to $x_{h+1}\sim P_h(\cdot\mid x_h,a_h)$. A (deterministic) nonstationary policy is a mapping $\pi:\mathcal S\times[H]\to\mathcal A$, with $\pi_h(x)\in\mathcal A$ the action prescribed at step $h$.\\
For any policy $\pi$, we define the value and action-value functions
\begin{align*}
    &V_h^\pi(x)\ :=\ \mathbb E\!\left[\sum_{h'=h}^H r_{h'}(x_{h'},\pi_{h'}(x_{h'}))\ \bigg|\ x_h=x\right],\\
&Q_h^\pi(x,a)\ :=\ r_h(x,a)+\mathbb E\!\left[\sum_{h'=h+1}^H r_{h'}(x_{h'},\pi_{h'}(x_{h'}))\ \bigg|\ x_h=x,\,a_h=a\right],
\end{align*}

and $V_h^\star(x):=\sup_\pi V_h^\pi(x)$. We adopt the notation $[P_h V](x,a):=\mathbb E_{x'\sim P_h(\cdot\mid x,a)}[V(x')]$ so that the Bellman equations read ($V^{\pi}_{H+1}\equiv 0$):
\begin{align}\label{eq:bellman_eq}
Q_h^\pi(x,a)=\big(r_h+P_h V_{h+1}^\pi\big)(x,a),\; V_h^\pi(x)=Q_h^\pi\!\big(x,\pi_h(x)\big)    
\end{align}

\paragraph{Reproducing Kernel Hilbert Spaces and Kernel ridge regression} We will assume that the state-action value function belongs to the reproducing kernel Hilbert space(RKHS) of a known kernel function $k$. Let $\mathcal{Z}=\mathcal{S}\times \mathcal{A}$ be the state-action product set. For a positive-definite function $k: \mathcal{Z}\times \mathcal{Z}\rightarrow \mathbb{R}$ . let RKHS $\mathcal{H}_k$ denote the Hilbert space of functions over $\mathcal{Z}$ equipped  with an inner product satisfying $(*)\forall z\in \mathcal{Z}, k(\cdot, z) \in \mathcal{H}_k$ and $(**) \forall z\in \mathcal{Z}, f \in \mathcal{H}_k,  \langle f, k(\cdot, z)\rangle =f(z)$. This inner product induces the norm defined as $\|f\|_{k}=\langle f,  f\rangle_{\mathcal{H}_k}$.  We introduce the shorthand notation for the feature function $\phi(z)=k(z, \cdot)$, assume without loss of generality $\|\phi(z)\|_{\cH_k} \leq 1$, and that all features lie in the unit ball of $\cH_k$.\\
Hilbert space $\cH_k$ also admits an alternate characterization in terms of eigenvalues of the kernel function $k$. Under mild conditions on $k$, Mercer's theorem provides a characterization of an orthonormal basis of $\cH_k$
\begin{theorem}\citep{SVM_Book}
Let $\mathcal{W}$ be a compact metric space and $k: \mathcal{Z} \times \mathcal{Z}\rightarrow \mathbb{R}$ be a continuous kernel. Furthermore, let $\Delta$ be a finite Borel probability measure supported on $\cZ$. Then, there exists an orthonormal system of functions $\{\psi_j\}_{j\in \mathbb{N}}$ in $\ell_2(\Delta,\mathcal{Z})$  and a sequence of non-negative values
$\{\lambda_j\}_{j\in \mathbb{N}}$ satisfying $\lambda_1\geq \lambda_2\dots \geq 0$, such that $k(z,z')=\sum_{j\in \mathbb{N}}\lambda_j\psi_j(z)\psi_j(z')$ holds for all $z,z'\in \mathcal{W}$ and the convergence is absolute and uniform over $z,z'\in \mathcal{Z}$. 
\end{theorem}
In particular $\forall f \in \cH_k$ and $\{\zeta_j\}_{j\in \mathbb{N}}=\{\sqrt{\lambda_j}\psi_j\}_{j \in \mathbb{N}}$ form an orthonormal basis of the space $\cH_k$. Behavior of eigen-values of $\{\lambda_j\}_{j\in \mathbb{N}}$ dictates the difficulty of many learning  problems over $\cH_k$. In this work, we consider kernels from the Mat\'ern kernel family. Next to the Squared-Exponential (SE) kernel\footnote{The methodology developed here is directly applicable to the SE kernel, as its eigenvalues decay exponentially.}, it is one of the most widely used kernel families and is especially attractive for function modeling, as it is norm-equivalent to appropriate order Sobolev spaces \citep{kanagawa1807gaussian}. Due to this norm equivalence, it is also studied in the wider context of smooth MDPs \cite{maran2024no, maran2024local}.

\begin{assumption}(eigen-decay and bounded $\|\psi_j\|_{\infty}$ assumption)\label{ass:eigen_assumption}
    Consider a kernel (covariance) function $k:\cZ\times \cZ \rightarrow \mathbb{R}$ and let $\{\lambda_i\}_{j\in \mathbb{N}}$ be its eigenvalues. Furthermore, suppose that compact $\cZ$ is contained within a $d$-dimensional unit hypercube. Kernel $k$ is said to be polynomially-decaying if $\forall j, \lambda_j\leq \ C_ p j^{-\beta_p}$. Additionally $\ell_2$ eigen functions $\{\psi_j\}_{j \in \mathbb{N}}$ of the kernel function $k$ have a bounded $\ell_{\infty}$ norm , $\forall j, \|\psi_j\|_{\infty}\leq F$.
\end{assumption}
Both the assumption on polynomial-eigen decay and $\ell_{\infty}$ bound on $\ell_2$-eigen function are crucial in obtaining tractable regret bounds and are thus almost ubiquitously made across bandit and kernel RL literature \citep{Vakili23-KerRL, vakili2021information, whitehouse2023sublinear,  yang2020function, salgia2023random, yeh2023sample}. For Matern  kernels both parameters $\beta_p$ have closed form expressions in terms of the smoothness parameter $\nu$, $\beta_p=\frac{2\nu}{d}+1$.\\
We briefly introduce Gaussian Processes (GPs), which we use to design 
frequentist estimators of the value function. A GP on the state-action 
domain $\cZ$ is a zero-mean random process with covariance function $k$, 
i.e.\ $\mathbb{E}[f(z)]=0$ and $\mathbb{E}[f(z)f(z')]=k(z,z')$. Given 
samples $\bfY_t$ of $f$ at points $\bfZ_t$, the posterior mean and variance 
are:
\begin{align}
    \mu_{t}(z) & =\bfk_{\mathbf{Z}_t}(z)^{\top}\left(\tau \mathbf{I}_t+
    \mathbf{K}_{\mathbf{Z}_t,\mathbf{Z}_t}\right)^{-1}\mathbf{Y}_t, 
    \label{eqn:posterior_mean}\\
    \sigma^2_t(z)& =k(z,z)-\bfk_{\mathbf{Z}_t}^{\top}(z)\left(\tau
    \mathbf{I}_t+\mathbf{K}_{\mathbf{Z}_t,\mathbf{Z}_t}\right)^{-1}
    \bfk_{\mathbf{Z}_t}(z), \label{eqn:posterior_variance}
\end{align}
where $\bfk_{\bfZ_t}=[k(z,z_1),\dots, k(z,z_t)]$ and 
$\bfK_{\bfZ_t,\bfZ_t}=\{k_{z_i,z_j}\}_{i,j=1}^{t}$. We model the $Q$ 
function as a GP sample, updating the posterior as noisy samples arrive; 
the posterior mean serves as our estimator, and the posterior variance 
naturally drives exploration. We refer the reader to 
Sec.~\ref{sec:app_B} for further details.

\paragraph{Kernel MDP assumption~\citep{Vakili23-KerRL, yang2020function}:}\label{assumption:ker_MDP}
We impose the following assumption on the parametric form of the reward function and the transition operator $\forall z\in\mathcal S\times\mathcal A,\ h\in[H]$:
\begin{align}\label{eq:kernel-mdp}
&r_h(z)=\langle \phi(z),\theta\rangle_{\cH_k}\quad P_h(s'\mid z)=\big\langle \phi(z),\Psi_h(s')\big\rangle_{\cH_k}
\end{align}
where $\forall s'\in \cS, \Psi_h(s'),\theta\in \cB_k(1)$, introducing the short-hand notation for $\cH_k$-ball, $\cB_k(r): =\{f\in \cH_k| \|f\|_{k}\leq r\}$ . Note that we assume the reward function is homogeneous throughout the episode, while the transition operator may vary. We assume throughout that this representation defines a valid transition kernel:
\(P_h(\cdot\mid z)\) is nonnegative and integrates to one for every $(z,h)$. Under the conditions in eq.(\ref{eq:kernel-mdp}) the state-action value function of an arbitrary policy $\pi$ is provably an element of $\cH_k$, moreover $Q^{\pi}_{h}\in \cB_k(H)$ \citep[Lemma 3]{yeh2023sample}.

\paragraph{Interaction protocol with preference feedback.}
Learning proceeds over $K$ \emph{comparison rounds}. At round $k\in[K]$, 
an initial state $x_1^k\in\mathcal{S}$ is chosen (possibly adversarially) 
and the learner selects a policy pair $(\pi,\pi')$, executing each from 
$x_1^k$ to generate trajectories
\begin{align*}
&\tau_k=(x_1^{k},a_1^{k},\ldots,x_{H+1}^{k}),\quad a_h^{k}=\pi(x^{k}_h);\\
&\tau'_k=({x'}_1^{k},{a'}_1^{k},\ldots,{x'}_{H+1}^{k}),\quad 
{a'}_h^{k}=\pi'(x'^{k}_h).
\end{align*}
No per-step rewards are observed; instead, the learner receives a single 
binary label
\begin{align}\label{eq:pref-model}
y_k\sim\mathrm{Ber}\!\left(\sigma\!\left(\sum_{h=1}^H r_h(z^{k}_h)
-\sum_{h=1}^H r_h(z'^{k}_h)\right)\right), \qquad 
\sigma(z):=\frac{1}{1+e^{-z}},
\end{align}
modeled via a Bradley--Terry--Luce logistic link on the difference of 
cumulative (unobserved) rewards. As in logistic bandits 
\citep{abeille2021instance,metelli2025generalized}, performance depends on 
the maximum inverse link gradient
$\kappa_{\cZ}=\max_{x\in[-H,H]}(\sigma'(x))^{-1}$,
which measures how identifiable the reward is from preference feedback: 
large $\kappa_{\cZ}$ corresponds to a flat link function and poor 
learnability.
\paragraph{Objective and regret.} The goal is to learn an optimal policy $\pi^\star$ (maximizing $V_1^\pi$) using only preference feedback \eqref{eq:pref-model}. Since each round requires two length-$H$ rollouts, the total number of environment steps after $K$ rounds is $ T=2KH$. We measure performance by the \emph{pairwise (weak) regret}
\begin{align}\label{eq:regret}
\mathfrak{R}(K):= \sum_{k=1}^{K}\left[V_1^\star(x_1^k)-\frac{\big(V_1^{\pi}(x_1^k)+\,V_1^{\pi'}(x_1^k)\big)}{2}\right]
\end{align}
which credits the learner with the average of the two policies it deploys in each comparison round. We aim to design an algorithm whose regret grows sublinearly in $K$  under the kernel MDP structure~\eqref{eq:kernel-mdp}.
%

\section{Algorithm Description}

In this section, we introduce our algorithm (\textbf{P}reference based \textbf{R}egularized \textbf{O}ptimization via \textbf{S}ampled \textbf{T}rajectory induced \textbf{O}ptimism ) \textsc{PROSTO}.
At the start of the $k^{\text{th}}$ episode, we estimate the state-action value function ${\widehat Q}^{k}_{h}$ based on the trajectory-level feedback from the previous $k-1$ episodes and the state transition samples. Due to the inherent difficulty of trajectory-level feedback, we employ separate approaches to estimate the reward and the next-state expected value function. At time $h$, we obtain the final estimate by  adding the individual estimates, $\widehat\theta_k+
\widehat\Psi_{h}\widehat V^{k}_{h+1}$. \\
By using the MDP parametrization in eq.(\ref{eq:kernel-mdp}), we can rewrite the difference of reward sums obtained by playing $\pi,\pi'$ in $k^{\text{th}}$ episode as $\sum_{h=1}^{H}\theta^{\top}\left(\phi\left(z^{k}_h\right)-\phi\left(z'^{k}_h\right)\right)$, where $z^{k}_h$ is a short-hand for for the state-action pair $(x^{k}_h,a^{k}_h)$.
 Along with eq.(\ref{eq:pref-model}), this parametrization
motivates the application of the maximum likelihood estimator (MLE) for the reward vector $\theta$, utilizing the preference feedback $\{y_j\}_{j=1}^{k-1}$ as the outcome variable and the trajectory differences $\sum_{h=1}^{H} \phi(z^{j}_h) - \phi(z'^{j}_h),j\leq k-1$ as the set of covariates. Specifically, an MLE incorporating a norm-weighted penalty is employed.
\begin{align}\label{eq:reward_theta_calc}
    \widehat\theta_{k}=\mathrm{argmin}_{\theta} &-\left(\sum_{i=1}^{k-1} y_i\log\sigma\left(\theta^{\top}\overline{\phi}_i\right)+(1-y_i)\log\left(1-\sigma\left(\theta^{\top}\overline{\phi}_i\right)\right)\right)+\tau\|\theta\|^2_{\cH_k}
\end{align}
where $\overline{\phi}_i= \sum_{h=1}^{H} \phi\left(z^{i}_h\right)-\phi\left(z'^{i}_h\right)\label{eq:trajectory_lvl_feature}$ is the trajectory difference feature and $\sigma(\cdot)$ is the  sigmoid function.
\begin{algorithm}[ht]
\caption{\textsc{PROSTO}}
\label{alg: algorithm}
\begin{algorithmic}[1]
    \STATE \textbf{Input}: Error probability $\delta_{\textsc{err}}$
    \FOR {$k \in \{1, 2,\dots K\}$}
    \STATE  Calculate $\widehat\theta_k$ from eq.(\ref{eq:reward_theta_calc}) 
    \STATE Generate 2 samples of exploration GPs, $\{\varepsilon_k, \varepsilon'_k\}$ from eq.(\ref{eq:noise_definition})
        \FOR {$h \in \{H, H-1,  \dots 1\}$}
            \STATE  Calculate $\widehat{Q}^{k}_h, \widehat{Q}'^{k}_h$ estimators from eq.(\ref{eq:Q_estimate}):
        \ENDFOR
        \FOR {$h \in \{1,2\dots H\}$}
            \STATE Play $a^{k}_h=\mathrm{argmax}_a \widehat Q^{k}_h(x^{k}_h, a) ,\; a'^{k}_h=\mathrm{argmax}_a \widehat Q'^{k}_h(x'^{k}_h, a)$ and collect transition samples.
        \ENDFOR
        \STATE  Collect preference-feedback label $y_k$
    \ENDFOR
  
\end{algorithmic}
\end{algorithm}

\paragraph{Exploration noise $\varepsilon_k$ and the reward estimator $\widehat\theta_{k}$.}
To design the estimator $\widehat \theta_{k}$ for the state-action value function, we will use the kernel logistic ridge regressor (see eq.(\ref{eq:reward_theta_calc}))(KLRR). However, as is common in RL applications, we need to ensure that the set of trajectories used to compute KLRR is sufficiently informative (i.e. exploratory). To incentivize trajectory-level exploration, we need to design an exploration bonus to be added to $\widehat\theta_{k}$. A difficulty in promoting exploration with trajectory-level feedback is that a commonly used elliptic (deterministic) bonus requires complete knowledge of the trajectory up front, at the beginning of each episode. This is clearly impossible without  fully knowing the transition dynamics, and to overcome this issue, we use a randomized (linear) exploration bonus \citep{cassel2024near, shani2020optimistic}.
In kernel-based RL, exploration noise is drawn from a Gaussian Process (GP). We draw $\varepsilon_k$ from a GP to promote exploration along the directions of those eigenfunctions of the empirical covariance operator $\bfW_k=\sum_{i=1}^{k-1}\overline{\phi}_i\otimes\overline{\phi}_i+\tau\mathbf{Id}$ that have small eigenvalues:
\begin{align}\label{eq:noise_definition}
    \varepsilon_k\sim \mathbf{GP}(0,\mathrm{Cov}(z_1,z_2)), \quad \mathrm{Cov}(z_1,z_2)=\beta^2_{\text{reward}}(\delta)\langle\phi(z_1),\phi(z_2)\rangle_{\bf{W}^{-1}_k}
\end{align}
 where $\beta_{\text{reward}}(\delta)$ is the confidence width defined in Lemma (\ref{Lemma:reward_confidence_bounds_app}). Intuitively, we identify underexplored directions and encourage play through these directions. Reward estimate with exploration (confidence) bonus is now simply calculated as $\widehat\theta_{k}+\varepsilon_k$.

\paragraph{Calculating the Value functions $\widehat{V}^{k}_h,\widehat{V}'^{k}_h$.}
As is usually the case in MDP's \citep{jin2020provably} we calculate the state-action and value functions recursively via value iteration. The difference in our setting is that the reward vector and the expected next-state value are estimated separately. As in eq.(\ref{eq:reward_theta_calc}), the reward estimate $\widehat{\theta}_k$ is calculated via kernel logistic regression, while the next state expected value function is calculated via the usual kernel ridge regression. In particular, for  $\pi$(and entirely analogous for $\pi'$): $\widehat{Q}^k_{H+1}=0$,
\begin{align}\label{eq:Q_estimate}
\widehat{Q}^{k}_h= \left[\widehat{\theta}_k+\varepsilon_{k}+\left(\sum_{i=1}^{k-1} \phi(z^i_h)\otimes\phi(z^i_h)+\lambda \bf{Id}\right)^{-1}\Phi^{h}_{k-1}[\widehat{V}^{i}_{h+1}(x^{i}_{h+1})]_{i\leq k-1}+b^{k}_{h}\right]_{\pm\beta_{\text{clip}}(H-h+1)}\nonumber
\end{align}
where $b^{k}_h(z)=\beta_{\text{trans}}(\delta, \varepsilon)\sqrt{\phi^{\top}(z)\left(\sum_{i=1}^{k-1} \phi(z^i_h)\otimes\phi(z^i_h)+\lambda \bf{Id}\right)^{-1}\phi(z)}$ is the deterministic elliptic potential used to incite optimism in estimating the expected next state value function and $\Phi^{h}_{k-1}=[\phi(z^{i}_h)]_{i\leq k-1}$ is the feature matrix. Here $\beta_{\text{trans}}(\delta, \epsilon)$ is the confidence width of the kernel ridge regression, which, unlike the bandit case, also depends on the covering mesh of the value function cover. Due to value iteration, confidence bounds need to hold over the entire value function class, thus the confidence width scales logarithmically with the size of the $\ell_{\infty}$ cover. Lastly, $[x]_{\pm C}$ is the clipping operator that projects $x$ into the interval $(-C, C)$. For an exact expression on $\beta_{\text{clip}}$\footnote{Note that as in \cite{cassel2024near}, the clipping interval is symmetric around zero. Because $\varepsilon_k$ is random, using one-sided clipping interval as in \cite{jin2020provably} would result in a constant $\cO(1)$ bias in the confidence bounds.}, please see the lemma(\ref{lemma: bounding_ell_infty}).

Both policies are designed greedily with respect to the state-action function, that is, the action is chosen so that the empirical state-action function is maximized $\widehat{V}^{k}_{h}(x)=\max_{a\in \cA}\widehat{Q}_h^{k}(x, a)$.

\paragraph{Tuning regularizers $\tau, \lambda$ and the covering mesh $\epsilon$.}  Tunning the norm penalty weights for both $\widehat\Psi^{k}_h\widehat V^{k}_{h+1},\widehat\theta_k$ transition and reward estimates is crucial to achieving sub-linear regret in our algorithm. In addition to allowing for more granular bias-variance trade-off \citep{whitehouse2023sublinear}, adjusting $\lambda,\tau$ in our work enables a sharp bound on the log-covering number. Reward-regularizer $\tau$ directly influences the RKHS norm of $\|\widehat\theta_k\|_{\cH_k}$ and the covering numbers of the sample paths $\varepsilon_k$(see Lemma.(\ref{lemma:gp_cover_noise})), which in part determine how complex a function class we need to cover to achieve uniform bounds over all value functions. Transition-regularizer $\lambda$ further ensures that the covering number of the elliptic bonus $b^{k}_h(\cdot)$  is not too large. Additionally, unlike the bandit setting, our setting requires tuning the covering accuracy $\epsilon$. If $\epsilon$  is too small, the covering number(see eq.(\ref{eq:main_text_cover})) is large and hence the confidence width $\beta_{\text{trans }}(\delta, \epsilon)$ is too wide thus making the estimator $\widehat Q^k_{h}$ inaccurate. On the other hand, if $\epsilon$ is too big, the value function cover is too coarse, leading again to an inaccurate estimator and vacuous regret bounds.
We tune the regret with respect to these three parameters essentially requiring a solution to a piece-wise linear 3-dimensional optimization problem(see. eq(\ref{eq:help_eq_14})):
\begin{align}
   \tau= H^2\log(K)^{\beta_p}K^{\frac{2}{\beta_p+1}}\quad \lambda= \log(K)^{\beta_p}K^{\frac{2}{\beta_p+1}}\quad \epsilon=H^{2}\kappa_{\cZ}K^{-\frac{\beta_p-1}{2(\beta_p+1)}}
\end{align}

Finally, we design policies $\pi, \pi'$ as greedy policies with respect to the estimated state-action value function. In particular, at the start of $k^{\text{th}}$ episode, each policy draws an independent exploration noise vector $\varepsilon_k$(from the distribution described in eq.(\ref{eq:noise_definition})) and calculates its own $Q$-estimate as in eq.(\ref{eq:Q_estimate}).

\section{Performance Guarantees and Technical Challenges}\label{sec:per_gaurantees_tech_challenge}

In this section, we state the cumulative regret upper bound of our algorithm $\textsc{PROSTO}$ and further expand on the challenges unique to learning with preference feedback in the kernel setting.
\begin{theorem}\label{Thrm:main_theorem}
    Consider the algorithm described in Algorithm.(\ref{alg: algorithm}), under Assumptions(\ref{ass:eigen_assumption}), (\ref{eq:kernel-mdp}) and with trajectory-preference based feedback described in eq.(\ref{eq:pref-model}). By setting the error probability as $\delta\leq 0.25\Phi(-1)$  the cumulative regret defined in eq.(\ref{eq:regret}), can be bounded with probability at least $1-\delta$ as: 
    \begin{align}\label{eq: final_reg_bound_main} 
          \mathfrak{R}(K)\leq \widetilde{\cO}\left(\kappa_{\cZ}\max\left(H^{2.5}K^{1-\frac{(\beta_p-1)^2}{2\beta_p(\beta_p+1)}},H^{3}K^{1-\frac{\beta_p-1}{2(\beta_p+1)}}\right)\right)
      \end{align}
    where $\kappa_{\cZ}$ is the maximum inverse gradient defined in Sec.(\ref{sec:problem}) and $\Phi(\cdot)$ is the cumulative distribution function of the standard normal variable.
\end{theorem}
 A detailed proof of Theorem.(\ref{Thrm:main_theorem}) can be found in Sec.(\ref{app:A}) in the Appendix.\\
For the remainder of this section, we highlight the technical difficulties unique to kernel MDPs and explain how our contributions address these challenges; see
Sec.~\ref{sec:discuss_appendix} for additional discussion.
\paragraph{Covering bounds for estimators in kernel  based MDPs.}\label{sec:difficulty_kernel_rl} 
Extending preference-based RL to kernel MDPs introduces technical obstacles absent in both linear preference-based RL and kernel RL with action-level rewards. The central challenge is controlling the log-covering number of the value-function class, which directly determines the widths of confidence intervals and is essential for optimism-based analyses~\citep{jin2020provably}.

In linear MDPs, the log-covering number is bounded polynomially in the ambient dimension, does not grow with the number of episodes, and only exhibits logarithmic growth with respect to the covering parameter $\epsilon$~\citep{jin2020provably}. Consequently, the covering mesh $\epsilon$ can be set to $1/\text{poly}(K)$ while the log-covering number remains $\widetilde{\cO}(1)$, leaving the regret order unchanged. In kernel settings, the situation is fundamentally different. Recall that the parameter $\beta_p = 2\nu/d + 1$ captures the smoothness of the kernel relative to the input dimension; the log-covering number of the value function class $\cV$  satisfies:~\citep{Vakili23-KerRL, yang2020function}:
\begin{align}\label{eq:main_text_cover}
    &\cV=\left\{\textsc{clip}\left(\max_{a\in \cA}\phi^{\top}(z)\theta+\beta\sqrt{\phi^{\top}(z)\bfB \phi(z)}, C_{\text{clip}}\right), \|\theta\|_{\cH_k}\leq R_{\text{norm}}, \|\bfB\|_2\leq \lambda^{-1}, \beta\leq B_{\text{optim}}\right\}\nonumber\\
    &\log \cN_{\infty}(\cV, \epsilon)= \cO\left(\left(\frac{R^{2}_{\text{norm}}}{\epsilon^2}\right)^{\frac{1}{\beta_p-1}}+\left(\frac{B_{\text{optim}}^2}{\lambda\epsilon^2}\right)^{\frac{2}{\beta_p-1}}\right)
\end{align}
The critical difficulty is that the exponent $1/(\beta_p - 1)$ diverges as $\beta_p \to 1$ (i.e., when $d \gg 2\nu$). Since the estimator norm $R_{\text{norm}}$, the confidence width $B_{\text{optim}}$, and the covering mesh $\epsilon$ are all polynomial in $K$, this means the log-covering number itself can grow as an arbitrary-degree polynomial in $K$, rendering standard regret analyses vacuous. As emphasized in  \cite{maran2024local}, it is exactly the settings where $\beta_p\sim 1$ that make it difficult to achieve no-regret bounds, even for action-level feedback.\\ 
Kernel preference feedback adds an additional layer of difficulty to bounding the log covering numbers; exploration necessitates the addition of a linear exploration bonus in the form of a GP(see eq.(\ref{eq:noise_definition})). This exploration noise is not sufficiently smooth and lies outside of the $\nu$-Mat\'ern RKHS w.p. 1 \citep{kanagawa1807gaussian}, making the bounds in eq.(\ref{eq:main_text_cover}) inapplicable. While sampling-based exploration in RL is certainly not new \citep{chowdhury2017kernelized}, value iteration poses new challenges for Thompson sampling. While  GP-TS \citep{chowdhury2019online} can simply disregard the complexity of the estimated/sampled posterior and only concentrate on the $\ell_{\infty}$ error bound, value iteration requires regularity of the estimators, implicitly through the covering bounds.

\paragraph{Related approaches.}The current state of the art for action-level kernel MDPs \citep{maran2024local} considers space partitioning in which the Bellman operator admits a low-error, locally linear approximation and avoids value iteration over a nonparametric functional family altogether. While \cite{maran2024local} achieves an impressive no-regret performance for a wide family of smooth MDPs, domain partitioning is fundamentally unapplicable in trajectory-level feedback.  The stochasticity of transitions prevents advance determination of which trajectory (and thus which part of the space) will be realized, precluding localized Bellman operator estimation. In addition, standard elliptic exploration bonuses cannot be computed causally; they require knowing the whole trajectory before the episode begins. 
An additional research direction relevant to preference learning with kernels is trajectory-level feedback learning in linear MDPs. Causality of the exploration bonus in linear MDP's is adressed by sampling Gaussian noise from the trajectory covariance~\citep{shani2020optimistic, cassel2024near}. Exploration in linear MDPs requires exploration only along $d$ directions, hence the addition of the exploratory noise does not alter the order of the log covering bounds in \cite{jin2020provably}.


\paragraph{Our technical contributions.}
The key to obtaining tight confidence bounds and meaningful regret guarantees is to simultaneously control the RKHS norm of the deterministic component of the $Q$-function estimator and the functional complexity of the sampled paths of the GP perturbation $\varepsilon_k$.  Our contribution addresses this through two approaches: (i) for KLRR and KRR estimators, this is done through RKHS and operator norm bounds; (ii) for GP perturbation, which does not have a finite norm in the $\nu$-Mat\'ern RKHS, we instead use a high-probability Gaussian covering argument.
\textbf{Tuning the reward regularizer $\tau$.} 
The central idea for controlling the KLRR estimator's RKHS norm is careful tuning of the regularizer $\tau$. Increasing $\tau$ places more weight on the norm penalty, yielding smaller-norm estimators, but creates a bias-variance trade-off. Critically, $\tau$ also controls the covering bounds of the exploration process; Intuitively, as shown in \citet{whitehouse2023sublinear}, tuning $\tau$ enables tighter bounds on the confidence width $\beta_{\text{reward}}(\delta)$, making the estimator more confident and thus allowing us to reduce exploration variance. However, the relationship between $\beta_{\text{reward}}(\delta)$ and $\tau$ is nonlinear; increasing $\tau$ excessively can worsen the confidence width of the KLRR estimator.

\textbf{Covering bounds for the posterior GP $\varepsilon_k$.}
As shown in \cite{kanagawa1807gaussian}, $\varepsilon_k$ does not possess sufficient smoothness to be in $\nu$-Mat\'ern RKHS and instead is a member of a more coarse RKHS, $\nu'<\nu-d/2$ Mat\'ern reproducing Hilbert space. As a result, the usual covering bounds in \cite{yang2020function} are inapplicable, and moreover, for $\beta_p<2$, the sampled GP does not belong to any Mat\'ern RKHS, thus necessitating an entirely different approach. To this end, we use the result of \cite{van2008rates}(Theorem 2.1) linking the log-covering bounds of the GP(with covariance $k(\cdot, \cdot)$)  sample paths and the probability that the GP stays within a($\ell_\infty$) ball(small ball probability). Intuitively, while the GP sample paths lie outside of any $\nu$-Mat\'ern RKHS
ball, with high probability they have a neighbor inside of a (reasonably sized) RKHS ball that is only $\epsilon$, away in $\ell_{\infty}$. Hence $\ell_{\infty}$, $\epsilon$-thickened RKHS ball contains almost all GP sample paths and has small metric entropy. With machinery from \cite{van2008rates}, we modify their result for the posterior GP $\varepsilon_k$. The posterior GP $\varepsilon_k$ is \textit{less random} than its prior $\mathbf{GP}(0, k)$ and thus we expect that the small ball probability for $\varepsilon_k$ is at least that of $\mathbf{GP}(0,k)$. This proves to be the case, and in a satisfying way, large $\tau$ necessitates less exploration noise while large $\beta_{\text{reward}}$ increases the noise. Indeed the log-covering bounds for the sample paths of $\varepsilon_k$ can be bounded (with high probability) as
\begin{align*}
   &\varepsilon_k\in G_{k, \epsilon,\delta} \quad \text{w.p.}\quad 1-\delta\\
   &\log \mathcal N_\infty(\mathcal G_{k, \epsilon,\delta},\epsilon)
    \le\widetilde{\mathcal O}\left(\left(\frac{\beta_r}{\sqrt{\tau}\epsilon}
        \right)^{\frac{2}{(\beta_p-1)}}+\log(1/\delta)\right).
\end{align*}

For more precise treatment of the covering bounds, please see Sec.(\ref{Sec:GP_stuff}).


\textbf{Tuning the transition regularizer $
\lambda$ and the covering parameter $\epsilon$.}
The transition regularizer $\lambda$ governs a bias-variance tradeoff 
analogous to $\tau$. The elliptic exploration bonus is defined through a norm induced by a $\lambda \mathbf{Id}$-shifted covariance operator, hence altering $\lambda$ directly alters the local behavior of the elliptic bonus and hence the covering bounds. A small $\lambda$ 
can provide highly confident(transition) estimator but inflates the covering number of the bonus; 
a large $\lambda$ reduces the covering number at the cost of overly 
pessimistic confidence intervals.
The covering mesh $\epsilon$ introduces a tradeoff that is unique to the kernel regime. In linear MDPs, the log covering number depends only polylogarithmically on $1/\epsilon$, so $\epsilon$ can be tuned freely without polynomial cost. In kernel MDPs, however, eq.~\eqref{eq:main_text_cover} shows that this dependence is polynomial in $1/\epsilon$, and can be of 
arbitrarily high degree in hard kernel instances ($\beta_p\to 1$). 
The two sides of the $\epsilon$ tradeoff are as follows. We run kernel regression on the nearest neighbor of $\widehat{V}^{k}_{h+1}$ in the 
$\ell_\infty$ cover rather than on $\widehat{V}^{k}_{h+1}$ itself, which 
introduces a misspecification error in learning the next-state value 
function $\Psi_h\widehat{V}^{k}_{h+1}$. This causes $\beta_{\text{trans}}$ 
to scale as $\widetilde{\mathcal{O}}(\epsilon\sqrt{K})$,%
\footnote{See eq.~\eqref{eq:transition_conf_width} for details.} 
so a large $\epsilon$ leads to trivially wide confidence intervals. 
Conversely, a small $\epsilon$ controls misspecification but implies
large bounds on the covering number. We tune $\epsilon$ to balance these two effects.

\section{Conclusion}

We study preference feedback in the kernel MDP setting and design a sublinear regret policy.  The main novelty of our approach is the derivation of the new GP-specific covering bounds and the tuning of the norm regularizer and the covering mesh to control the covering number of the value function class and, hence, the confidence width of the $Q$-function estimator. \\
Whether the regret bound of $\textsc{PROSTO}$ can be made order-optimal remains an open question. Domain partitioning 
approaches~\citep{janz2020bandit,valko2013finite}, which currently provide the only known path to order-optimal regret in kernel MDPs and contextual bandits, are incompatible with trajectory-level feedback, so a fundamentally different technique would be required.

\bibliographystyle{abbrvnat}
\bibliography{references}
\newpage
\onecolumn
\section{Appendix A.}\label{app:A}


The proof is conceptually split into three parts: GP prerequisite lemmas, estimation and optimism. For the GP part, we bound the covering bounds of the GP perturbation noise $\varepsilon_k$ (Lemma.(\ref{Lemma: cover_noise})) and also provide uniform bounds on the value of the GP function at feature functions (Lemma.(\ref{lemma:unifom_bound_function})). These are later utilized in deriving Lemmas.(\ref{lemma:conf_bound_trans_lemma_app}, \ref{lemma:bounding_estimation_noise_app})

For the estimation part, Lemma.(\ref{Lemma:reward_confidence_bounds_app}) provides the confidence bounds for the reward function estimate $\widehat\theta_r$ and Lemmas.(\ref{lemma: bounding_ell_infty}) and (\ref{Lemma:RKHS_bound_val_func}) are crucial in providing a bound on the log-covering number of the state value function, and subsequently proving the confidence bounds of the transition operator in Lemma.(\ref{lemma:conf_bound_trans_lemma_app}). Finally, in Lemma.(\ref{lemma:bounding_estimation_noise_app}) we bound the effect of the GP exploration noise $\varepsilon_k$(see eq.(\ref{eq:noise_definition})) on the estimation of the $Q$ -function.\\
On the optimism side, we show constant probability($\approx \Phi(-1)$) "approximate optimism" in Lemma.(\ref{lemma:optimism_lemma_app}) and then extend this to $\approx 1-o(1)$ probability optimism in Lemma.(\ref{lemma:regre_decomp}). Finally, in the Theorem.(\ref{Thrm:main_theorem}), we present the final regret upper bound.\\

\subsection{Axillary lemmas on the posterior GP }\label{Sec:GP_stuff}

We first derive log covering bound for the exploratory GP process sampled in eq.(\ref{eq:noise_definition}), and initially we concetrate on the derivation for the "vanila" Mat\'ern GP, $w$. The idea is that while the sampled GP, $w \sim \mathbf{GP}(0, k(\cdot, \cdot))$ is not in RKHS of $\nu$-Ma\'tern kernel, it can be shown that with high probability that $\varepsilon_k$ is within an $\epsilon$ (in the $\ell_{\infty}$ sense) of a reasonably small ball in RKHS of $\nu$-Mat\'ern kernel. In other words , with high probability we will show that:
$$
w \in M\cdot\{f\in \cH_k| \|f\|_{\cH_k}\leq 1\}+ \epsilon\cdot \{f\in \ell_\infty |\|f\|_{\infty}<1\}
$$

Hence, even though $\epsilon_k \in M\cdot\{f\in \cH_k| \|f\|_{\cH_k}\leq 1\}$ w.p. 0, the log covering number of the GP sample paths is, up to a constant the same as that of $B_{\cH_k}(1)$. Exactly this is shown in \citep{van2008rates}(Theorem 2.1). Combining the said result with the small ball probability bounds in \cite{van2011information} (Lemma 3) we have the following result:

\begin{lemma}\citep{van2008rates, van2011information}\label{Lemma: cover_noise}
    Consider a Ma\'tern GP , $w\sim \mathbf{GP}(0,k(\cdot, \cdot))$ sampled over a  compact domain $\cZ$. There exists a measurable set $\cG_{\epsilon, \delta}\subset \ell_{\infty}(\cZ)$ so that, $w\in \cG_{\epsilon, \delta}$ with probability $1-\delta$ and:
    \begin{align*}
    \log \cN_{\infty}(\cG_{\epsilon, \delta}, \epsilon)\leq \cO
    \left(\log(1/\delta)+\epsilon^{-\frac{d}{\nu}}\right)
    \end{align*}
\end{lemma}

We next generalize Lemma.(\ref{Lemma: cover_noise}) to the posterior GP $\varepsilon_k\sim \mathbf{GP}(0, \mathrm{Cov}_{k}(\cdot, \cdot))$. The idea is that as we sample new trajectory pairs, the uncertainty in the point estimate provably drops and thus  $k_{\text{posterior}}\preceq k$ in the positive-definite sense. This essentially implies that $\varepsilon_k$ is "less random" than $w$.  The subtlety here is that $\varepsilon_k$ is inflated by $\beta_{\text{reward}}$ and compressed by the regularizer $\tau$. We thus expect the covering bound to be trade-off between these two quantities.

\begin{lemma}[Posterior GP covering bound]\label{lemma:gp_cover_noise}
For the posterior $\varepsilon_t$ GP sampled from eq.(\ref{eq:noise_definition}) there exists a measurable set $\cG_{k,\epsilon, \delta}$ so that with probability at least $1-\delta$, $\varepsilon_t \in \cG_{t,\epsilon, \delta},\forall t\leq K $ and :
\begin{align*}
    \log\cN_{\infty}(\cG_{k,\epsilon, \delta})\leq \cO\left(\left(\frac{\beta_r}{\sqrt{\tau}\epsilon}\right)^{\frac{d}{\nu}}+\log(K/\delta)\right)
\end{align*}
\end{lemma}

\begin{proof}
The idea is to use the domination $k_{\text{posterior}}\preceq k$ in order to utilize Anderson's set inequality \cite{anderson1955integral}. The posterior  covariance can be written as: 
\begin{align*}
    \mathrm{Cov}(x,y)&=\beta^2_{r}\phi(x)^{\top}\left(\sum_{i=1}^{t-1}\overline{\phi}_i\otimes \overline{\phi}_i+\tau\right)^{-1}\phi(y)=\\
    &=\frac{\beta^2_r}{\tau}\underbrace{\left(k(x,y)- \overline{k}_t(x)^{\top}(\overline{\bfK}+\tau\bfI)^{-1}\overline{k}_{t}(y)\right)}_{k_{\text{post}}(x,y)}
\end{align*}

Where $\overline{k}_t(\cdot)=\begin{bmatrix}\sum_{h=1}^{H}k(z,z^{1}_h)-k(z,z'^{1}_h)\\
         \sum_{h=1}^{H}k(\cdot,z^{2}_h)-k(\cdot,z'^{2}_h)\\
         \dots\\
         \sum_{h=1}^{H}k(\cdot,z^{1}_h)-k(\cdot,z'^{1}_h)
         \end{bmatrix}$. We now have $k_{\text{post}}\preceq k$  from positive-definiteness of $(\overline{\bfK}_{t}+\tau\bfI)$ of the trajectory difference Gramian. We hence indeed have $\mathrm{Cov}(\cdot, \cdot)\preceq k(\cdot, \cdot)\frac{\beta^2_r}{\tau}$. We can now use Anderson's inequality to derive small-ball probabilities for the posterior GP $\varepsilon_k$ and then use \cite{van2008rates} to finish the argument. Indeed, with GP measure of $\varepsilon_t$ being conditioned on the sigma algebra $\cF_{t-1}$ spanned by the previous episodes, we have:
         \begin{align*}
             \mathbb P\!\left(
        \|\varepsilon_t\|_\infty\le r \right) \ge \mathbb P\!\left( \frac{\beta_r}{\sqrt{\tau}}\|w\|_\infty\le r \right).
         \end{align*}
        Thus by using small ball probability bound for the $\nu$-Mat\'ern GP from \cite{van2011information} we have:
        \begin{align*}
            -\log\mathbb P\!\left(
        \|\varepsilon_t\|_\infty\le \epsilon \right)\leq \cO\left(\left(\frac{\beta_r}{\sqrt{\tau}\epsilon}\right)^{\frac{d}{\nu}}\right)
        \end{align*}
        Direct application of Theorem 2.1 in \cite{van2008rates} now gives the desired bound.
\end{proof}

Next, we derive general uniform bounds on the feature-GP functional $\langle \varepsilon_k, \phi(z)\rangle \forall \cZ$. The idea is to use the H\"older continuity $(\alpha=\nu)$ of the Mat\'ern feature function along with \cite{dudley1967sizes} integral to bound the supremum of the GP in terms of the log-covering number.

\begin{lemma}[Uniform covariance normalized bounds]
\label{lemma:unifom_bound_function}
 Conditionally on the $\sigma$-algebra spanned by the $k-1$ episodes \(\mathcal F_{k-1}\), we can parametrically write the posterior GP as:
\[
    \langle \phi(z),\varepsilon_k\rangle
    =
    \beta_r G(W_k^{-1/2}\phi(z)),
\]
\footnote{Here \(\langle \phi(z),\varepsilon_k\rangle\) should be understood as a Gaussian
linear functional, not necessarily as the inner product with an
\(\mathcal H_k\)-valued random element.}where \(G\) is an isonormal Gaussian process on \(\mathcal H_k\). Then with
probability at least \(1-\delta\), simultaneously for all \(z\in\mathcal Z\),
\[
    |\langle \phi(z),\varepsilon_k\rangle|
    \le
    \beta_r\|\phi(z)\|_{W_k^{-1}}
    \left[
        C
        \sqrt{
            \frac d\alpha
            \log\!\left(
                C L_\phi
            \sqrt{\frac{kH^2+\lambda}{\lambda}}\right)}
        +
        \sqrt{2\log(1/\delta)}
    \right],
\]
where \(C>0\) is a universal constant.
\end{lemma}

\begin{proof}
All probability measures here are assumed to be conditioned on \(\mathcal F_{k-1}\). Define
\[
    Y_z:=
    \frac{\langle \phi(z),\varepsilon_k\rangle}{\beta_r\|\phi(z)\|_{W_k^{-1}}}.
\]
 $Y=(Y_z)_{z\in\mathcal Z}$ is a centered unit-variance Gaussian process. Its canonical metric is
\[
    d_Y(z,z')
    :=
    \left(\mathbb E|Y_z-Y_{z'}|^2\right)^{1/2}
    =
    \|u_z-u_{z'}\|_{\mathcal H_k}.
\]
where, $ u_z :=
    \frac{W_k^{-1/2}\phi(z)}
    {\|W_k^{-1/2}\phi(z)\|_{\mathcal H_k}}$. From Mat\'ern feature H\"older continuity  and $\lambda\mathbf{Id}\preceq\bfW_k\preceq (kH^2+ \tau) \mathbf{Id}$ we have:
\[
    d_Y(z,z')
    =
    \|u_z-u_{z'}\|
    \le
    2L_\phi \sqrt{\frac{kH^2+\lambda}{\lambda}}
    \|z-z'\|_2^\alpha .
\]

where $L_{\phi}$ is the feature H\"older constant. Denote as $C(k,\lambda)= \sqrt{\frac{kH^2+\lambda}{\lambda}}$ Since \(\mathcal Z\subset[0,1]^d\),
\[
    N(\mathcal Z,d_Y,\eta)
    \le
    \left(
        \frac{\left(C_{k,\lambda}L_\phi\right)^{1/\alpha}}{\eta^{1/\alpha}}
    \right)^d
\]

By Dudley's entropy integral for centered Gaussian processes,
\[
    \mathbb E\sup_{z\in\mathcal Z}|Y_z|
    \le
    C
    \int_0^2
    \sqrt{
        \log\bigl(2N(\mathcal Z,d_Y,c\eta)\bigr)
    }
    \,d\eta .
\]
The elementary bound $\int_0^2 \sqrt{a+b\log(L_{\phi}/\eta)}\,d\eta\le C\sqrt{a+b\log(CC_{k, \lambda}L_{\phi})}$
gives
\[
    \mathbb E\sup_{z\in\mathcal Z}|Y_z|
    \le
    C\sqrt{
        \frac d\alpha
        \log(CC_{k,\lambda}L_{\phi})
    }.
\]

Finally, apply the Borell--TIS \citep{adler2007random} inequality to the signed process
\[
    \widetilde Y_{(z,s)}:=sY_z,
    \qquad
    (z,s)\in\mathcal Z\times\{-1,+1\}.
\]
Since \(\sup_{z,s}\widetilde Y_{(z,s)}=\sup_z|Y_z|\) and $\sup_{z,s}\mathrm{Var}(\widetilde Y_{(z,s)})=1$

Borell--TIS gives, with probability at least \(1-\delta\),
\[
    \sup_{z\in\mathcal Z}|Y_z|
    \le
    \mathbb E\sup_{z\in\mathcal Z}|Y_z|
    +
    \sqrt{2\log(1/\delta)}.
\]
Combining the last two equations and multiplying by
\(\beta_r\|\phi(z)\|_{W_k^{-1}}\) gives the result.
\end{proof}

\subsection{Main Proof}

\begin{lemma}\cite{pmlr-v119-faury20a,metelli2025generalized}(eq.(9, 11))\label{Lemma:reward_confidence_bounds_app}
    Consider a kernel logistic bandit setting where environment feedback is sequentially generated as $y_t\sim \mathrm{Bernoulli}(\sigma(\overline{\phi}^{\top}_t\theta))$. For the maximum likelihood estimator $\widehat\theta_t$ in eq.(\ref{eq:reward_theta_calc}) and the reward-trajectory difference data-set  $\{\overline{\phi}_i, y_i\}_{i=1}^{t-1}$ the following confidence bound holds with probability at least $1-\delta$:
    \begin{align}\label{eq:conf_bounds_reward_app}
        \Gamma_1:\;\|\theta-\widehat\theta_t\|_{\bfW_t}\leq \underbrace{3H\kappa_{\cZ}\left(2\sqrt{2\log\left(\frac{1}{\delta}\right)+\log\det\left(\tau^{-1}\bfW_t+\mathbf{Id}\right)}+\sqrt{\tau}\right)}_{\beta_r(\delta)}
    \end{align}
    , where $\bfW_t$ is the covariance operator induced by the trajectory-difference features(see sec.(\ref{sec:app_B}) for details).
\end{lemma}
\begin{proof}
    The bound is a special case of a sharper confidence bound(wrt $\kappa_{\cZ}$) in \cite{metelli2025generalized} and follows immediately from eq. (9,11) applied to the trajectory difference kernel defined in Sec.(\ref{sec:app_B}).
\end{proof}

\begin{lemma}($\ell_{\infty}$ \label{lemma: bounding_ell_infty}bound)
   For the exploration process $\varepsilon_k$ (eq.(\ref{eq:noise_definition})), and the reward vector estimate $\widehat \theta_k$(eq.(\ref{eq:reward_theta_calc})), the following inequality  holds with probability at least $1-\delta$ we have:
    \begin{align*}
        \Gamma_2: \; \max_{z\in \cZ}\left|\phi(z)^{\top}(\widehat\theta_k+\Psi_h \widehat V^{k}_{h+1}+\varepsilon_k)\right|
        \leq (H-h+1)\beta_{\textsc{clip}}(\delta)
    \end{align*}
    where $\beta_{\textsc{clip}}(\delta)=\left(3+ \frac{3\beta_r(\delta)}{\sqrt{\tau}} \sqrt{\log\left(\frac{2}{\delta}\right)+\frac{2d}{\min(\nu,1)}\log(K)}\right)$
\end{lemma}

\begin{proof}
We can expand the inequality as:
\begin{align*}
    \left|\phi^{\top}(z)\left(\widehat\theta_k+\Psi^{k}_h\widehat V^{k}_{h+1}+\varepsilon_k\right)\right|&\leq \|\phi(z)\|_{\bfW^{-1}_k}\|\theta-\widehat\theta_k\|_{\bfW_k}+\left|\phi^{\top}\theta+\mathbb{E}_{x\sim P_h(\cdot |z)}\left[\widehat V^{k}_{h+1}\right]\right|+\left|\phi^{\top}(z)\varepsilon_k\right|\leq \\
    &\leq  1+\frac{\beta_r(\delta)}{\sqrt{\tau}}+\left|\mathbb{E}_{x\sim P_h(\cdot |z)}\left[\widehat V^{k}_{h+1}\right]\right|+\left|\phi^{\top}(z)\varepsilon_k\right|
\end{align*}

For the first term note that we have  $\bfW_k\succeq \tau \mathbf{Id}$ and hence $\|\phi(z)\|_{\bfW^{-1}_k}\leq \sqrt{1/\tau}$ and thus immediately  $\|\phi(z)\|_{\bfW^{-1}_k}\|\theta-\widehat\theta_k\|_{\bfW_k}\leq \frac{\beta_r}{\sqrt{\tau}}$. By Assumption.(\ref{assumption:ker_MDP}), $
\|\theta\|_{\cH_k}=1$ and thus $|\phi^{\top}(z)\theta|\leq \|\phi(z)\|_{\cH_k}\|\theta\|_{\cH_k}\leq 1$.\\
It remains to bound the term $\phi^{\top}(z)\varepsilon_k$. To this end, we directly apply the Lemma.(\ref{lemma:unifom_bound_function}):
\begin{align*}
    |\langle \phi(z),\varepsilon_k\rangle|
    \le
    \beta_r\|\phi(z)\|_{W_k^{-1}}C\left(\sqrt{\frac d\alpha\log\!\left(C L_\phi\sqrt{\frac{kH^2+\lambda}{\lambda}}\right)}+\sqrt{2\log(1/\delta)}\right)
\end{align*}

Where the last inequality follows from the GP noise $\varepsilon_k$ norm bound in Lemma.(\ref{Lemma:RKHS_bound_val_func}).
Lastly, by the clipping in the \textsc{PROSTO} we have $\widehat V^{k}_{h+1}(x)=\max_{a\in \cA}\widehat Q^{k}_{h+1}(x,a)\leq \beta_{\textsc{clip}}(\delta)(H-h)$ and thus by combining the previously derived inequalities we have:
\begin{align*}
     &\left|\phi(z)^{\top}(\hat\theta_k+\Psi^{k}_h \widehat V^{k}_{h+1}+\varepsilon_k)\right|\leq \\
     & \leq\left(3+\frac{\beta_r(\delta)}{\sqrt{\tau}}+ \frac{2\beta_r(\delta)}{\sqrt{\tau}} C\left(\sqrt{\frac d\alpha\log\!\left(C L_\phi\sqrt{\frac{kH^2+\lambda}{\lambda}}\right)}+\sqrt{2\log(1/\delta)}\right)\right)(H-h+1)\leq\\
     &\leq \underbrace{\left(3+ \frac{3\beta_r(\delta)}{\sqrt{\tau}} \sqrt{\log\left(\frac{2}{\delta}\right)+\frac{2d}{\alpha}\log(K)}\right)}_{{\beta_{\textsc{clip}}(\delta)}}(H-h+1)
\end{align*}

\end{proof}

\begin{lemma}\label{Lemma:RKHS_bound_val_func}(bounding the Hilbert norm of the linear, deterministic part of $\widehat Q^{k}_{h}$) For Gaussian process defined in the eq.(\ref{eq:noise_definition}) and maximum likelihood estimate of the reward vector in eq.(\ref{eq:reward_theta_calc}), in $k^{\text{th}}$-episode the following bound on the  RKHS norm holds with probability at least $1-\delta$ :
\begin{align*}
\Gamma_3:\; h\leq H\quad &\left\|\widehat\theta_{k}+\widehat\Psi^{k}_{h}\widehat V_{h+1}^{k}\right\|_{\cH_k}\leq\\
&\leq \sqrt{\frac{\kappa_{\cZ}}{3\tau}\left(\log\left(\frac{1}{\delta}\right)+\overline\gamma_k(\tau)\right)}+2\beta_{\textsc{CLIP}}(\delta)H+\frac{2\beta_{\textsc{CLIP}}(\delta)H}{\sqrt{\tau}}\sqrt{2\left(1+\gamma_{k}(\tau)+\log(kH/\delta)\right)}
\end{align*}
We can also bound the norm asymptotically as:

\begin{align*} \left\|\varepsilon_k+\widehat\theta_{k}+\widehat\Psi^{k}_{h}\widehat V_{h+1}^{k}\right\|_{\cH_k}\leq \underbrace{\widetilde\cO\left(\sqrt{\frac{\kappa_{\cZ}\overline{\gamma}_k(\tau)}{3\tau}}+\beta_{\textsc{CLIP}}(\delta)H\max\left(1,\frac{1}{\sqrt{\tau}}\right)\sqrt{\gamma_{k}(\tau)}\right)}_{C_{\cH}}
\end{align*}

\end{lemma}
\begin{proof}
    We will bound the norm of each summand individually and then complete the proof using the triangle inequality with respect to the Hilbert norm.\\
    \textbf{Bounding $\|\widehat\theta_k\|_{\cH_k}$}. We use the mean value theorem to connect the difference $\theta-\widehat\theta_k$ to the standard self-normalizing martingale, which is frequently encountered in the bandit literature \citep{abbasi2011improved}. Following identities in \cite{metelli2025generalized}(Lemma C.3 and eq.(138))\footnote{Also follows from an immediate extension of the mean value argument in \cite{pmlr-v119-faury20a}} we have:
    \begin{align}
        &(\theta-\widehat\theta_k)\bfG_k(\theta,\widehat\theta_k)=\left(\sum_{i=1}^{k-1} \overline\phi_i\eta_i\right)+\tau\theta \label{eq:approx_diff}\\
        &\text{where, }\quad \bfG_k(\theta,\widehat\theta_k)\succeq \frac{3}{\kappa_{\cZ}} \bfW_k\nonumber
    \end{align}
    here $\eta_i=y_i-\sigma\left(\overline\phi^{\top}_i\theta\right)$ is the  noise induces by the preference feedback. Note that $|\eta_i|<2$ hence $\eta_i$ is $1$-sub-Gaussian. Continuing the derivation in eq.(\ref{eq:approx_diff}) we have:
    \begin{align*}
        \left\|\theta-\widehat\theta_k\right\|_{\cH_k}&\leq \left\|\bfG^{-1}_k\left(\theta,\widehat\theta_k\right)\left(\sum_{i =1}^{k-1}\overline\phi_i\eta_i\right)\right\|_{\cH_k}+\left\|\theta\right\|_{\cH_k}\leq \\
        &\leq \sqrt{\frac{\kappa_{\cZ}}{3\lambda}}\left\|\sum_{i=1}^{k-1}\overline{\phi}_i\eta_i\right\|_{\bfW^{-1}_k}+1
    \end{align*}
    where the first inequality follows from $\bfW_k\succeq\tau\bf{Id}$ and the second is a direct application of eq.(\ref{eq:approx_diff}).
    To bound the first summand, we use the concentration bound on the self-normalizing martingale derived in \cite{whitehouse2023sublinear}. This gives us a probability at least $1-\delta$ :
    \begin{align*}
        \left\|\sum_{i=1}^{k-1}\overline{\phi}_i\eta_i\right\|_{\bfW^{-1}_k}\leq \sqrt{\log\left(\frac{1}{\delta}\right)+\overline \gamma_k(\tau)}
    \end{align*}
    where $\overline \gamma_k(\tau)$ is the information gain defined by the trajectory level features(see. Lemma.(\ref{lemma:diff_kernel_gain_bound})). We thus finally have w.p. at least $1-\delta$:
    \begin{align*}
        \left\|\widehat\theta_k\right\|_{\cH_k}\leq \sqrt{\frac{\kappa_{\cZ}}{3\tau}\left(\log\left(\frac{1}{\delta}\right)+\overline\gamma_k(\tau)\right)}+1
    \end{align*}

    \textbf{Bounding $\left\|\widehat\Psi^{k}_{h}\widehat V^{k}_{h+1}\right\|_{\cH_k}$.}
    By Lemma 5 in \cite{Vakili23-KerRL} and Lemma.(\ref{lemma: bounding_ell_infty}) we have:\begin{align*}\left\|\widehat\Psi^{k}_{h}\widehat V^k_{h+1}\right\|_{\cH_k}\leq \|\Psi^{k}_h\widehat V^{k}_{h+1}\|_{\cH_k}+ \frac{2\beta_{\textsc{CLIP}}(\delta)}{\sqrt{\tau}}\sqrt{2\left(1+\gamma_{k}(\tau)+\log(1/\delta)\right)}
    \end{align*} 
    By \cite{yeh2023sample}(Lemma. 3) and the clip condition in Alg.(\ref{alg: algorithm}) we have $\|\Psi_h\widehat V^{k}_{h+1}\|_{\cH_k}\leq \beta_{\textsc{CLIP}}H$. Lastly, we use a union bound over all $H$ transitions within the episode to obtain :
    \begin{align*}\left\|\widehat\Psi^{k}_{h}\widehat V^k_{h+1}\right\|_{\cH_k}\leq \beta_{\textsc{CLIP}}(\delta)H+ \frac{2\beta_{\textsc{CLIP}}(\delta)}{\sqrt{\tau}}\sqrt{2\left(1+\gamma_{k}(\tau)+\log(H/\delta)\right)}
    \end{align*} 

\end{proof}
 Define the set of value functions in episode $k$ as :
    \begin{align}\label{eq:val_set}
        \cV_k=&\left\{\textsc{clip}\left(\max_{a\in \cA}\phi^{\top}(s,a)(\theta+\epsilon)+\beta\sqrt{\phi^{\top}(s,a)\bfB \phi(s,a)}, \pm(H-h+1)\beta_{\textsc{clip}}\right)\right\},\\
        &\text{ where, }\|\theta\|_{\cH_k}\leq C_{\cH}, \|\bfB\|_2\leq \lambda^{-0.5}, \beta\leq C_{\beta}\sqrt{\lambda}, \varepsilon \in \cG_{k ,\epsilon, \delta} \nonumber
    \end{align}

Where $\cG_{k,\epsilon, \delta}$ is the high-probability set defined in Lemma.(\ref{lemma:gp_cover_noise}).
 Of central importance in deriving the confidence bounds for the estimator $\widehat\Psi^{k}_{h}V, V \in \cV_k$ is the $\epsilon$-covering number $\cN_{\infty}(\cV_k, \epsilon)$. In the next lemma, we derive exact dependence on the confidence width of the KRR  estimator $\widehat{\Psi}^{k}_hV$ as a function of regularizers $\lambda, \tau$ and the covering mesh $\epsilon$. This will be crucial later on, when we tune the regret with respect to $\lambda, \tau, \epsilon$.
 
\begin{lemma}\label{lemma:conf_bound_trans_lemma_app}(Confidence bounds for the transition estimate $\widehat\Psi^{k}_h$) For the confidence parameter bounded as $\beta_t(\delta)\leq C_{\beta}\sqrt{\lambda}$ for some suitable constant $C_{\beta}$ we have with probability at least $1-\delta$:
    \begin{align*}
         \Gamma_4:\left\|\left(\widehat\Psi^{k}_{h}-\Psi_{h}\right) V\right\|_{\mathrm{Cov}^{k}_h(\lambda)}\leq \underbrace{4\beta_{\textsc{clip}}\left(\frac{\delta}{2}\right)H\sqrt{\gamma_{K}(\lambda)+2\log\left(\frac{2}{\delta}\right)+2\log\cN_{\infty}(\cV_k, \varepsilon)}+\sqrt{\lambda}H+2\epsilon\sqrt{K}}_{\beta_t(\delta)}
    \end{align*}
    , where $\mathrm{Cov}^{k}_h(\lambda)= \sum_{i\leq k-1}\phi(z^{i}_h)\otimes\phi(z^{i}_h)+\lambda\mathbf{Id}$ is the empirical covariance induced by the state-action pairs played in the $h^{\text{th}}$ instance in each episode and $\cN_{\infty}(\cV_k, \epsilon)$ is the size of $\epsilon$-covering of the value function family $\cV_k$.  In particular, as a function of $\lambda, \tau, \epsilon$ the confidence bound can be written as :
     \begin{align}
      &\left\|\left(\widehat\Psi^{k}_{h}-\Psi_{h}\right) V\right\|_{\mathrm{Cov}^{k}_h(\lambda)}\leq\label{eq:transition_conf_width}\\
      & \leq \widetilde{\cO}\left(H^2\kappa_{\cZ}\sqrt{\left(\frac{\overline{\gamma}_K(\tau)}{\tau}\lor1\right)}\sqrt{\left(\left(\frac{\overline{\gamma}_K(\tau)}{\tau}\lor1\right)\frac{1}{\epsilon^2}\right)^{\frac{1}{\beta_p-1}}+\left(\frac{1}{\epsilon^2}\right)^{\frac{2}{\beta_p-1}}+\gamma_K(\lambda)}+\sqrt{\lambda}H+\epsilon\sqrt{K}\right)\nonumber
    \end{align}
\end{lemma}
\begin{proof}
    We apply the self-normalizing martingale concentration inequality in \cite{whitehouse2023sublinear}(Theorem 1) to the noise generated as $\zeta(\overline{V}, z^{k}_{h})=\overline{V}(x^{k}_{h+1})- \phi^{\top}(z^{k}_{h})\Psi_{h}\overline{V}$  where $\overline{V}$ is an arbitrary element in the cover $\overline{V}\in \cV_{\infty}(\epsilon)$ . By definition of $\cV$ and Lemma(\ref{lemma: bounding_ell_infty}), the noise due to using one- point estimate of the value function average can be bounded as $\forall z^k_{h} \; \left|\zeta(\overline{V}, z^{k}_{h})\right|\leq 2\beta_{\textsc{clip}}(\delta)H$ . Now, directly applying  the self-normalizing martingale concentration with probability  at least $1-\delta$, we have:
    \begin{align}\label{eq:self_normalize_ineq}
        \forall \overline{V}\in \cV, \left\|\sum_{i \leq k-1} \phi(z^{i }_h)\zeta(\overline{V}, z^{i}_h)\right\|_{\mathrm{Cov}^{k}_h(\lambda)^{-1}}\leq  4\beta_{\textsc{clip}}\left(\frac{\delta}{2}\right)H\sqrt{\gamma_{K}(\lambda)+2\log\left(\frac{2}{\delta}\right)+2\log\cN_{\infty}(\cV,\epsilon)}
    \end{align}
    To now obtain the bound with respect to $V$ we will bound the error due to using $\overline{V}$ instead of $V$. Denote the error by $\Delta V=\overline{V}-V$ we can now write:
    \begin{align*}
        \left\|\left(\widehat\Psi^{k}_{h}-\Psi_{h}\right) V\right\|_{\mathrm{Cov}^{k}_h(\lambda)}&\leq \left\|\lambda \mathrm{Cov}^{k}_h(\lambda)^{-1}\Psi_{h}V\right\|_{\mathrm{Cov}^{k}_h(\lambda)}+\left\|\mathrm{Cov}^{k}_h(\lambda)^{-1}\sum_{i \leq k-1} \phi(z^{h}_i)\zeta(V, z^{h}_i)\right\|_{\mathrm{Cov}^{k}_h(\lambda)}\leq \\
        &\leq \lambda \underbrace{\|\Psi_{h}V\|_{\mathrm{Cov}^{k}_h(\lambda)^{-1}}}_{(\star)}+\\
        &+\underbrace{\left\|\sum_{i\leq k-1} \phi(z^{i}_h)\zeta(\overline{V}, z^{i}_h)\right\|_{\mathrm{Cov}^{k}_h(\lambda)^{-1}}}_{(\star\star)}+\underbrace{\left\|\sum_{i\leq k-1} \phi(z^h_i)\zeta(\Delta{V}, z^h_i)\right\|_{\mathrm{Cov}^{k}_h(\lambda)^{-1}}}_{(\star\star\star)} 
    \end{align*}
    Note that by definition of a cover $\forall z\in \cZ, \zeta(\Delta V,z)\leq 2\varepsilon$ and by definition of the covariance operator $\mathrm{Cov}^{k}_h(\lambda)$ we also have $\sum_{i\leq k-1}\phi(z^{i}_{h})\otimes\phi(z^{i}_{h})\prec \mathrm{Cov}^{k}_h(\lambda)$ thus:
    \begin{align*}
        (\star\star\star)\leq \sqrt{\sum_{z\in c}\zeta^2(\Delta{V}, z)}\leq 2\epsilon\sqrt{K}
    \end{align*}
    To bound $(\star)$ note that $\|\Psi_{h}V\|_{\cH_k}\leq H$  and thus $(\star)\leq\sqrt{\lambda}H$. Finally, the bound for $(\star\star)$ follows immediately from eq.(\ref{eq:self_normalize_ineq}) and thus we finally have w.p. at least $1-\delta$:
    \begin{align*}
        \left\|\left(\widehat\Psi^{k}_{h}(c)-\Psi_{h}\right) V\right\|_{\mathrm{Cov}^{k}_h(\lambda)}\leq 4\beta_{\textsc{clip}}\left(\frac{\delta}{2}\right)H\sqrt{\gamma_{K}(\lambda)+2\log\left(\frac{2}{\delta}\right)+2\log\cN_{\infty}(\cV, \varepsilon)}+\sqrt{\lambda}H+2\epsilon\sqrt{K}
    \end{align*}
     Next, we instantiate the covering bounds as a function of $\lambda, \tau, \epsilon$.  We bound the covering number $\cN_{\infty}(\cV,\epsilon)$ of the value function class by using the derivation in \cite{yang2020function} and Lemma.(\ref{lemma:gp_cover_noise}), with the distinction that we require a bound with a greater regularizer $\lambda$-granularity. More specifically, efficiently bounding the covering number of the space of the bonus functions with respect to $\lambda$ will be crucial. Define the set of elliptical bonuses  with the bounded operator norm over the RKHS as $\cB=\left\{b |b=\sqrt{\phi^{\top}(s,a)\bfB \phi(s,a)}, \|\bfB\|_2\leq \lambda \right\}$
    By  Lemma(D.3) in \cite{yang2020function}\footnote{The dependency of the log-covering number on the Hilbert norm regularizer, $\lambda$ , in \cite{yang2020function} is collapsed in a constant. The dependency stated here follows immediately from derivation in equations (E.42, E.43) in \cite{yang2020function}} and the assumption(\ref{ass:eigen_assumption}) on eiegen-decay we have:
    \begin{align*}
        \log \cN_{\infty}(\cB, \varepsilon)\leq \cO\left(\left(\frac{1}{\lambda\varepsilon^2}\right)^{\frac{2}{\beta_p-1}}\right) 
    \end{align*}
    Under the condition $\beta_t(\delta)\leq C_{\beta}\sqrt{\lambda}$ stated in this Lemma  and the result of the Lemma.(\ref{lemma:gp_cover_noise}) log covering bound becomes:
    \begin{align*}
        \log\cN_{\infty}(\cV, \epsilon)= \cO\left(\left(\frac{\beta^2_r}{\tau\epsilon^2}\right)^{\frac{1}{\beta_p-1}}+\left(\frac{1}{\epsilon^2}\right)^{\frac{2}{\beta_p-1}}\right)
    \end{align*}
     Plugging in the RKHS bound for the linear  part of the estimator $\widehat{Q}^{k}_{h}$ derived in lemma(\ref{Lemma:RKHS_bound_val_func}) into the equation above we can finalize the confidence bounds as:
       \begin{align*}
      &\left\|\left(\widehat\Psi^{k}_{h}-\Psi_{h}\right) V\right\|_{\mathrm{Cov}^{k}_h(\lambda)}\leq\\
      & \leq \widetilde{\cO}\left(H^2\kappa_{\cZ}\sqrt{\left(\frac{\overline{\gamma}_K(\tau)}{\tau}\lor1\right)}\sqrt{\left(\left(\frac{\overline{\gamma}_K(\tau)}{\tau}\lor1\right)\frac{1}{\epsilon^2}\right)^{\frac{1}{\beta_p-1}}+\left(\frac{1}{\epsilon^2}\right)^{\frac{2}{\beta_p-1}}+\gamma_K(\lambda)}+\sqrt{\lambda}H+\epsilon\sqrt{K}\right)
    \end{align*}
\end{proof}

\begin{corollary}\label{corollary:pointwise_approx}
    For $\forall k\leq K, h\leq H$ and $\forall z\in \cS\times \cA$ we have w.p. $1-\delta$:
    \begin{align}
        \left|\Psi^{k}_h\widehat{V}^{k}_{h+1}(z)-\widehat\Psi^{k}_{h}\widehat{V}^{k}_{h+1}(z)\right|\leq \beta_{t}\left(\delta\right)\sqrt{\phi^{\top}(z)\left(\sum_{i=1}^{k-1} \phi(z^i_h)\otimes\phi(z^i_h)+\lambda \bf{Id}\right)^{-1}\phi(z)}+ 2\epsilon
    \end{align}
    , where $\epsilon$ is the covering mesh of $\cV_{\infty}(\epsilon)$.
\end{corollary}
In other words, in addition to widening the confidence width by $\epsilon\sqrt{K}$ we also pay a $2\epsilon$ additive approximation error for using the $\cV_{\infty}(\epsilon)$ covering. Next, we show that an "approximate optimism" holds for the approximate value function of the policy $\pi$, $\widehat V^{\pi'}$, with constant probability for a fixed episode up to a factor of $\epsilon$.
\begin{lemma}\label{lemma:optimism_lemma_app}
    Let $\widetilde V^{\pi'}_1(x^{k}_1)=\mathbb{E}_{\pi'}\left[\sum_{h=1}^{H}\phi(x^k_h, a^k_h)^{\top}\left(\widehat\theta_k+\varepsilon_k\right)\right]$ with probability at least $\Phi(-1)-2\delta$ we have:
    \begin{align}\label{eq:approx_optimism_app}
        \Gamma_5:\quad\forall x^{k}_1, V_1^{\star}(x^{k}_{1})-\widehat V^{\pi}_{1}(x^{k}_{1})+\widetilde V^{\pi'}_1(x^{k}_{1})-V^{\pi'}_1(x^{k}_{1})\leq 2\epsilon H
    \end{align}
Where $\Phi$ is the cumulative density function(CDF) of a standard normal variable $\sim \cN(0, 1)$ and $\epsilon$ is the mesh of the $\cV_{\infty}(\epsilon)$ cover.
\end{lemma}
\begin{proof}
    We  start by using the value function difference decomposition from \cite{shani2020optimistic}:
    \begin{align*}
        V^{\star}_1(x^k_1)-\widehat V^{\pi}_1(x^{k}_1)&=
        \underbrace{\mathbb{E}_{\pi^{\star}}\left[\sum_{h=1}^{H} \widehat Q^{k}_{h}(x^{k}_h,\pi^{\star}(x_h))-\widehat Q^{k}(x^{k}_h, \pi(x^{k}_h))\right]}_{(*)}+\\
        &+\underbrace{\mathbb{E}_{\pi^{\star}}\left[\sum_ {h=1}^{H}\phi^{\top}(x^{k}_h,\pi^{\star}(x^{k}_h))(\theta+\Psi^{k}_h\widehat V^{\pi}_{h+1})-\widehat Q^{k}(x^{k}_h, \pi^{\star}(x^{k}_h))\right]}_{(**)}
    \end{align*}
    Note that $\pi$ is a greedy policy with respect to $\widehat Q^{k}_{h}$ and thus $(*)\leq 0$. Before  moving on to $(**)$ we use Lemma(\ref{lemma: bounding_ell_infty}) to transform $\widehat Q^k_ h$ in more manageable form:
    \begin{align}
        \widehat Q^{k}_h(z)&=\bigg[\phi^{\top}(z)\left(\widehat\theta_k+\varepsilon_k+\widehat\Psi^{k}_h\widehat V^{\pi}_{h+1}\right)+b^{k}_h(z)\bigg]_{\beta_{\textsc{clip}}(\delta)(H-h+1)}=\nonumber\\
        &=\bigg[\phi^{\top}(z)\left(\widehat\theta_k+\varepsilon_k+\Psi^{k}_h\widehat V^{\pi}_{h+1}\right)+\phi^{\top}(z)\left(\widehat\Psi^{k}_h-\Psi^{k}_h\right)\widehat V^{\pi}_{h}+b^{k}_h(z)\bigg]_{\beta_{\textsc{clip}}(\delta)(H-h+1)}\geq \nonumber\\
        &\underbrace{\geq}_{\text{Corollary.(\ref{corollary:pointwise_approx})}} \bigg[\phi^{\top}(z)\left(\widehat\theta_k+\varepsilon_k+\Psi^{k}_h\widehat V^{\pi}_{h+1}\right)+2\epsilon\bigg]_{\beta_{\textsc{clip}}(\delta)(H-h+1)}=\nonumber\\
        &=\phi^{\top}(z)\left(\widehat\theta_k+\varepsilon_k+\Psi^{k}_h\widehat V^{\pi}_{h+1}\right)+2\epsilon\label{eq:help_eq_6}
    \end{align}
    The inequality follows from Lemma.(\ref{lemma:conf_bound_trans_lemma_app}) and the monotonicity of the clip operator and the last equality follows from Lemma.(\ref{lemma: bounding_ell_infty}). Note that, for the inequality derived in eq.(\ref{eq:help_eq_6}) to hold, we need $\Gamma_2\cap\Gamma_4$ to hold; thus, the inequality holds w.p. at least $1-2\delta$.\\
    To now provide a bound for $(**)$,  we plug in the inequality derived in eq.(\ref{eq:help_eq_6}) :
    \begin{align*}
        &{\mathbb{E}_{\pi^{\star}}\left[\sum_ {h=1}^{H}\phi^{\top}(x^{k}_h,\pi^{\star}(x^{k}_h))(\theta+\Psi_h\widehat V^{\pi}_{h+1})-\widehat Q^{k}_h(x^{k}_h, \pi^{\star}(x^{k}_h))\right]}\leq\\
    &\leq \mathbb{E}_{\pi^{\star}}\left[\sum_{h=1}^{H} \phi^{\top}(x^{k}_h,a^{\star}_h))\left(\theta-\widehat \theta_k-\varepsilon_k\right)\right]+2\epsilon H=\\
    &=\mathbb{E}_{\pi^{\star}}\left[\sum_{h=1}^{H} \phi^{\top}(x_h,a^{\star}_h))\right]\left(\theta-\widehat\theta_k-\varepsilon_k\right)+2\epsilon H
    \end{align*}
    The last equality follows as $\varepsilon_k,\widehat\theta_k$ are independent of the transition dynamics in the $k^{\text{th}}$ episode. For the same reason, we can rewrite the second difference as:
    \begin{align*}
        \widetilde V^{\pi'}_1(x^{k}_{1})-V^{\pi'}_1(x^{k}_{1})=\mathbb{E}_{\pi'}\left[\sum_{h=1}^{H}\phi^{\top}(x^k_h, \pi'(x^k_h))\right]\left(\widehat\theta_k+\varepsilon_k- \theta\right)
    \end{align*}
    Plugging this back into the equation above we have:
    \begin{align*}
        &V_1^{\star}(x^{k}_{1})-\widehat V^{\pi}_{1}(x^{k}_{1})+\widetilde V^{\pi'}_1(x^{k}_{1})-V^{\pi'}_1(x^{k}_{1})\leq\\
        &\leq \underbrace{\left(\mathbb{E}_{\pi^{\star}}\left[\sum_{h=1}^{H} \phi^{\top}(x^k_h,a^{\star}_h))\right]-\mathbb{E}_{\pi'}\left[\sum_{h=1}^{H}\phi^{\top}(x^k_h, \pi'(x^k_h))\right]\right)}_{\alpha}\left(\theta-\widehat\theta_k-\varepsilon_i\right)+ 2\epsilon H\leq \\
        &\leq  \beta_r(\delta)\|\alpha\|_{\bfW^{-1}_k}- \alpha^{\top}\varepsilon_k+2\epsilon H
    \end{align*}
Where the third line is the consequence of the confidence bound Lemma(\ref{Lemma:reward_confidence_bounds_app}). By using the parametric from of the distribution  of $\varepsilon_k$ in  Lemma.(\ref{lemma:unifom_bound_function}) we have $\alpha^{\top}\varepsilon_k\sim\cN(0,\beta^2_r(\delta)\|\alpha\|^2_{\bfW^{-1}_k})$. Thus, by the union bound with the event in eq.(\ref{eq:help_eq_6}), we have that the event in eq.(\ref{eq:approx_optimism_app}) holds w.p. at least $\Phi(-1)-2\delta$.
\end{proof}

After obtaining approximate optimism, we turn to bounding the estimation error $\widehat V_1^{\pi}-V_1^{\pi}$. In particular, in the $i^{\text{th}}$ episode we need to bound the product $\overline{\phi}^{\top}_i\varepsilon_i$, where $\overline{\phi}_i$ is the trajectory difference defined in eq.(\ref{eq:reward_theta_calc}). Usual sub-Gaussian concentration bounds are inapplicable here as $\overline{\phi}_i, \varepsilon_i$ are not independent, given that policy $\pi$ samples all actions during the $i^{\text{th}}$ episode $\{a^{i}_h\}_{h=1}^{H}$ based on the exploration noise $\varepsilon_i$.\\
To resolve this issue, we use uniform bounds that hold across the entire trajectory-difference space. 
\begin{lemma}\label{lemma:bounding_estimation_noise_app}
    Consider the Gaussian process $\varepsilon_k$ (eq.(\ref{eq:noise_definition})) generated at the start of $k^{\text{th}}$ episode  and two arbitrary trajectories $\{z_i\}_{i=1}^{H},\{z'_i\}_{i=1}^{H}\in \mathbb{R}^{dH}$. With probability at least $1-\delta$  we have :
    \begin{align}\label{eq:bounding_eestimation_noise_app}
        \forall\{z_h\}_{h=1}^{H},\{z'_h\}_{h=1}^{H}\; &\left|\left(\sum_{h=1}^{H} \phi\left(z_h\right)-\phi\left(z'_h\right)\right)^{\top}\varepsilon\right|\leq\\
        &\leq \beta_r(\delta)\left\|\left(\sum_{h=1}^{H} \phi\left(z_h\right)-\phi\left(z'_h\right)\right)\right\|_{\bfW^{-1}_k} C
        \sqrt{
            \frac{dH}{\alpha}
            \log\!\left(
                C L_\phi
            \sqrt{\frac{kH^2+\lambda}{\lambda}}\right)}
    ,\nonumber
    \end{align}
    for brevity we introduce $\beta_a(\delta)=\beta_r(\delta)C
        \sqrt{
            \frac{dH}{\alpha}
            \log\!\left(
                C L_\phi
            \sqrt{\frac{kH^2+\lambda}{\lambda}}\right)}$
\end{lemma}
\begin{proof}

We apply exactly the same reasoning as in Lemma(\ref{lemma:unifom_bound_function}) to the domain $\cZ^{2H}$ and the trajectory difference feature defined in App.(\ref{sec:app_B}) $\overline{\phi}(\tau, \tau')=\sum_{z\in \tau, z'\in \tau'} \phi(z)-\phi(z')$.
\end{proof}

We now extend the optimism Lemma (\ref{lemma:optimism_lemma_app}) over all episodes by following the derivation analogous to \cite{zanette2020frequentist,wu2023making}. The idea is that the algorithm is optimistic every $\approx\frac{1}{\Phi(-1)}$ episodes, and that, for the remaining time, we can bound the optimism error by conditioning on the high-probability bound in Lemma. (\ref{lemma:bounding_estimation_noise_app}).\\
First define the "worst case" random process $\varepsilon^{\dag}_i$ as :
\begin{align}\label{eq:worst_case_noise_app}
    \varepsilon^{\dag}_i=\mathrm{argmin}_{\varepsilon \in \cT}\widehat V^{\pi}(x^i_1,\varepsilon)-\widetilde{V}^{\pi'}_{1}(x^{i}_1,\varepsilon)
\end{align}
where $\cT$ is the set of all functions over $\cZ$ satisfying noise-estimation bound eq.(\ref{eq:bounding_eestimation_noise_app}). Also, define the "worst-case" value functions as :
\begin{align}\label{eq:worst_case_val_app}
    \widehat V^{\dag,\pi}(x^{i}_1)= \widehat V^{\pi}\left(x^i_1,\varepsilon^{\dag}_i\right) \qquad \widetilde V^{\dag,\pi'}(x^{i}_1)= \widetilde V^{\pi'}\left(x^i_1,\varepsilon^{\dag}_i\right)
\end{align}
Note that the value functions $\widehat V^{\dag,\pi}(x^{i}_1),\widetilde V^{\dag,\pi'}(x^{i}_1)$ are no longer random functions with respect to the random process $\varepsilon_i$ as they are now parametrized by the(deterministic) "worst-case" process $\varepsilon^{\dag}_i$.
\begin{lemma}(Extending the optimism to the high probability regime) \label{lemma:regre_decomp} For the worst case value functions defined in eq.(\ref{eq:worst_case_val_app}) we have with probability at least $1-\delta$:
\begin{align}\label{eq:regret_decompoposition_app}
      &\sum_{i=1}^{K}V_1^{\star}(x^{i}_{1})-\widehat V^{\pi}_{1}(x^{i}_{1})+\widetilde V^{\pi'}_1(x^{i}_{1})-V^{\pi'}_1(x^{i}_{1})\leq\\
      &\leq \frac{1}{(\Phi(-1)-2\delta)}\sum_{i=1}^{K}\left(\left(\widehat V^{\pi}_{1}(x^{i}_{1})-\widetilde V^{\pi'}_1(x^{i}_{1})\right)+\left(-\widehat V^{\dag,\pi}(x^{i}_1)+\widetilde V^{\dag,\pi'}(x^{i}_1)\right)\right)+\\
      &+6\sqrt{K}H\beta_{\textsc{clip}}(\delta)\sqrt{\log\left(\frac{1}{\delta}\right)}+2\epsilon KH \nonumber 
\end{align}    
\end{lemma}
\begin{proof}
    Note that as a direct consequence of Lemma.(\ref{lemma:optimism_lemma_app}) we have:
    \begin{align*}
        V_1^{\star}(x^{i}_{1})-V^{\pi'}_1(x^{i}_{1})\leq \mathbb{E}_{\varepsilon_i}\left [\widehat V^{\pi}_{1}(x^{i}_{1})-\widetilde V^{\pi'}_1(x^{i}_{1})\bigg|\Gamma_5\right]+2H\epsilon
    \end{align*}
    Furthermore, by the definition of worst-case value functions, we can write:
    \begin{align}\label{eq:help_eq_1}
        &V_1^{\star}(x^{i}_{1})-\widehat V^{\pi}_{1}(x^{i}_{1})+\widetilde V^{\pi'}_1(x^{i}_{1})-V^{\pi'}_1(x^{i}_{1})\leq \\
        &\mathbb{E}_{\varepsilon_i}\left [\widehat V^{\pi}_{1}(x^{i}_{1})-\widetilde V^{\pi'}_1(x^{i}_{1})-\left(\widehat V^{\dag,\pi}(x^{i}_1)-\widetilde V^{\dag,\pi'}(x^{i}_1)\right)\bigg| \Gamma_5\right]+2\epsilon H \nonumber
    \end{align}
    where the worst-case value functions are non-random with respect to the exploration process $\varepsilon_i$; thus, they are non-random with respect to the event $\Gamma_5$, spanned by the process. We can thus freely include the expression within the conditional expectation. By total conditioning theorem on $\Gamma_5$ we also have:
    \begin{align}\label{eq:help_eq_2}
         &2\epsilon H+ \mathbb{E}_{\varepsilon_i}\left [\widehat V^{\pi}_{1}(x^{i}_{1})-\widetilde V^{\pi'}_1(x^{i}_{1})-\left(\widehat V^{\dag,\pi}(x^{i}_1)-\widetilde V^{\dag,\pi'}(x^{i}_1)\right)\right]\geq\\
         &\geq \mathbf{P}(\Gamma_5)\mathbb{E}_{\varepsilon_i}\left[2\epsilon H+\widehat V^{\pi}_{1}(x^{i}_{1})-\widetilde V^{\pi'}_1(x^{i}_{1})-\left(\widehat V^{\dag,\pi}(x^{i}_1)-\widetilde V^{\dag,\pi'}(x^{i}_1)\right)\bigg| \Gamma_5\right]+\nonumber\\
         &+\mathbf{P}(\overline{\Gamma_5})\underbrace{\mathbb{E}_{\varepsilon_i}\left [\widehat V^{\pi}_{1}(x^{i}_{1})-\widetilde V^{\pi'}_1(x^{i}_{1})-\left(\widehat V^{\dag,\pi}(x^{i}_1)-\widetilde V^{\dag,\pi'}(x^{i}_1)\right)\bigg| \overline{\Gamma_5}\right]}_{\geq 0}\geq\nonumber \\
         &\geq \mathbf{P}(\Gamma_5)\left(V_1^{\star}(x^{i}_{1})-\widehat V^{\pi}_{1}(x^{i}_{1})+\widetilde V^{\pi'}_1(x^{i}_{1})-V^{\pi'}_1(x^{i}_{1})\right)  \nonumber \text{, by eq.(\ref{eq:help_eq_1})} 
    \end{align}

    Lastly to remove the expectation operator $\mathbb{E}_{\varepsilon_i}[
    \cdot]$ we note that the sequence:
    \begin{align*}
        \left\{\mathbb{E}_{\varepsilon_i}\left [\widehat V^{\pi}_{1}(x^{i}_{1})-\widetilde V^{\pi'}_1(x^{i}_{1})\right]-\left(\widehat V^{\pi}_{1}(x^{i}_{1})-\widetilde V^{\pi'}_1(x^{i}_{1})\right)\right\}_{i=1}^{K}
    \end{align*}
    is a martingale difference sequence absolutely bounded as $4H\beta_{\textsc{clip}}(\delta)$ due to the Lemma.(\ref{lemma: bounding_ell_infty}). Due to Azuma-Hoefding, we have the bound w.p. at least $1- \delta$:
    \begin{align*}
        \left|\sum_{i=1}^{K} \mathbb{E}\left [\widehat V^{\pi}_{1}(x^{i}_{1})-\widetilde V^{\pi'}_1(x^{i}_{1})\right]-\left(\widehat V^{\pi}_{1}(x^{i}_{1})-\widetilde V^{\pi'}_1(x^{i}_{1})\right)\right|\leq 6\sqrt{K}H\beta_{\textsc{clip}}(\delta)\sqrt{\log\left(\frac{1}{\delta}\right)} 
    \end{align*}
    Adding eq.(\ref{eq:help_eq_2}) over $K$ episodes and using the result of Lemma.(\ref{lemma:optimism_lemma_app}) gives the desired inequality.
\end{proof}

We are now ready to state and prove the regret bound of Algorithm(\ref{alg: algorithm})\\

\begin{theorem}\label{Thrm:main_theorem_app}
    Consider the algorithm described in Algorithm.(\ref{alg: algorithm}), under Assumptions(\ref{ass:eigen_assumption}), (\ref{eq:kernel-mdp}) and with trajectory-preference based feedback described in eq.(\ref{eq:pref-model}). By setting the error probability sufficiently small $\delta\leq 0.25\Phi(-1)$\footnote{Any $\delta<0.5\Phi(-1)-O(1)$ would suffice.} the cumulative regret defined in eq.(\ref{eq:regret}), can be bounded with probability at least $1-\delta$ as:
    \begin{align}\label{eq: final_reg_boud} 
        \mathfrak{R}(K)\leq \widetilde{\mathcal{O}}\left((\kappa_{\cZ}\max\left(H^{2.5}K^{1-\frac{(\beta_p-1)^2}{2\beta_p(\beta_p+1)}},H^{3}K^{1-\frac{\beta_p-1}{2(\beta_p+1)}}\right)\right)
    \end{align}
    , where $\kappa_{\cZ}$ is the maximum inverse gradient defined in Sec.(\ref{sec:problem}).
\end{theorem}
\begin{proof}
    We will only bound the regret incurred by the first policy $\pi$ and then double the regret bound and half the probability of error $\delta$. This is sufficient since $\pi'$ is an i.i.d. copy of $\pi$, so the same regret guarantees apply to $\pi'$ with the same error probability. We proceed with the proof in 2 steps. We will hold off tuning the regularizers $\lambda, \tau$, and the mesh cover $\varepsilon$ and  first develop a general asymptotic regret expression in these values.\\

    We first rewrite the regret expression in a more convenient way that allows the use of the Lemma. (\ref{lemma:optimism_lemma_app}):
    \begin{align}
       &\sum_{k=1}^{K}V_1^\star(x_1^k)-V_1^{\pi}(x_1^k)=\label{eq:regret_1}\\
           =&\sum^{K}_{k=1}\underbrace{\left(V_1^{\star}(x^{i}_{1})-\widehat V^{\pi}_{1}(x^{i}_{1})+\widetilde V^{\pi'}_1(x^{i}_{1})-V^{\pi'}_1(x^{i}_{1})\right)}_{(\dag)}+\underbrace{\left(\widehat V^{\pi}_{1}(x^{i}_{1})-V^{\pi}_{1}(x^{i}_{1})+V^{\pi'}_1(x^{i}_{1})-\widetilde V^{\pi'}_1(x^{i}_{1})\right)}_{(\dag\dag)}\nonumber
    \end{align}

    \noindent\textbf{Step 1.} Bounding $(\dag)$.\\
    
    Recall by Lemma(\ref{lemma:regre_decomp}) we have:
    \begin{align*}
        &(\dag)\leq \widetilde\cO\left(\epsilon KH\right)+\frac{1}{\Phi(-1)-2\delta}\sum_{i=1}^{K}\left(\left(\widehat V^{\pi}_{1}(x^{i}_{1})-\widetilde V^{\pi'}_1(x^{i}_{1})\right)+\left(-\widehat V^{\dag,\pi}(x^{i}_1)+\widetilde V^{\dag,\pi'}(x^{i}_1)\right)\right)
    \end{align*}
    We can further expand the second term as:
    \begin{align}\label{eq:help_eq_4}
        &\sum_{i=1}^{K}\left(\widehat V^{\pi}_{1}(x^{i}_{1})-\widetilde V^{\pi'}_1(x^{i}_{1})\right)+\left(-\widehat V^{\dag,\pi}(x^{i}_1)+\widetilde V^{\dag,\pi'}(x^{i}_1)\right)=\\
        &=\sum_{i=1}^{K}\underbrace{\left(\widehat V^{\pi}_{1}(x^{i}_{1})-V_1^{\pi}(x^{i}_1)+V^{\pi'}(x^{i}_1)-\widetilde V^{\pi'}_1(x^{i}_{1})\right)}_{(\ddag)}+\underbrace{\left(V_1^{\pi}(x^{i}_1)-\widehat V^{\dag,\pi}(x^{i}_1)+\widetilde V^{\dag,\pi'}(x^{i}_1)-V^{\pi'}(x^{i}_1)\right)}_{(\ddag\ddag)}\nonumber
    \end{align}
    Note that $(\ddag)\equiv(\dag\dag)$ and thus it is sufficient to bound $(\ddag),(\ddag\ddag)$.
    \begin{align*}
        &(\ddag)= \sum_{k=1}^{K}V^{\pi'}(x^{i}_1)-\widetilde V^{\pi'}_1(x^{i}_{1})+\\
        &+\sum_{k=1}^{K}\mathbb{E}_{\pi}\left[\sum_{h=1}^{H}\widehat Q^{k}_ {h}(x^k_h,\pi(x^k_h))-\widehat Q^{k}_{h}(x^k_h,a^{\star \pi k}_h))\right]
        +\mathbb{E}_{\pi}\left[\sum_{h=1}^{H} \widehat Q^{k}_{h}(x^k_h,\pi(x^k_h))-\phi^{\top}(x^k_h,\pi(x^k_h))\left(\theta+\Psi_{h}\widehat V^{k}_{h+1}\right)\right]
    \end{align*}
    Where the identity is the decomposition due to \cite{shani2020optimistic}(Lemma 1) applied to the greedy policy of the state-action value estimate $\widehat Q^{k}_h(\cdot)$. Note that due to the greedy nature of the policy $\pi$ we have $\pi(x^k_h)\equiv {a}^{\star\pi k}_{h}$ thus the second summand is zero. We will now further bound the third summand.\\
    Using the parametric form of $\widehat Q^{k}_{h}$ in eq.(\ref{eq:Q_estimate}) we can further write:
    \begin{align*}
        &\mathbb{E}_{\pi}\left[\sum_{h=1}^{H} \widehat Q^{k}_{h}(z^{k}_h)-\phi^{\top}(z^{k}_h)(\theta+\Psi_{h}\widehat V^{k}_{h+1}))\right]=\\
        &=\mathbb{E}_{\pi}\left[\sum_{h=1}^{H} \textsc{clip}\left(\phi^{\top}(z^{k}_h)\left(\widehat\Psi^{k}_h\widehat V^{k}_{h+1}+\widehat\theta_k+\varepsilon_{k}+b^{k}_h(z^{k}_h)\right), \pm\beta_{\textsc{clip}}(H-h+1) \right)-\phi^{\top}(z^{k}_h)(\theta+\Psi_{h}\widehat V^{k}_{h+1}))\right]\leq\\
        &\leq \mathbb{E}_{\pi}\left[\sum_{h=1}^{H} -\phi^{\top}(z^{k}_{h})\left(\theta-\hat\theta_k-\varepsilon_{k}\right)+\phi^{\top}(z^{k}_{h})\left(\hat\Psi^{k}_h(z^{k}_{h})-\Psi^k_h\right)\widehat V^k_{h+1}+b^{k}_h(z^k_h)\right]
    \end{align*}

    For the ease of notation, we introduce a shorthand for the state-action pair $z^k_h:=\left(x^k_h,\pi(x^k_h)\right)$. Note that in the second line we used the inequality $\textsc{clip}\left(x,\pm\beta_{\textsc{clip}}(H-h+1)\right)\leq x, \forall x\geq -\beta_{\textsc{clip}}(\delta)(H-h+1)$. \\
    We can now use the confidence bounds from the Corollary.(\ref{corollary:pointwise_approx}) as:
    \begin{align*}
        \left|\phi^{\top}(z^{k}_h)\left(\widehat\Psi^{k}_h-\Psi^k_h\right)\widehat V^k_{h+1}\right|&\leq \left\|\phi(z)\right\|_{\mathrm{Cov}_h^{k}( \lambda)^{-1}}\left\|\left(\widehat\Psi^{k}_h-\Psi^k_h\right)\widehat V^k_{h+1}\right\|_{\mathrm{Cov}_h^{k}( \lambda)^{-1}}+2\epsilon\leq \\
        &\leq \beta_t(\delta)\left\|\phi(z)\right\|_{\mathrm{Cov}^{k}_h(\lambda)^{-1}}+2\epsilon\equiv b^{k}_h(z)+2\epsilon
    \end{align*}
    Plugging this into the  equation above, we have:
    \begin{align*}
        &\sum_{k=1}^{K}\mathbb{E}_{\pi}\left[\sum_{h=1}^{H} \widehat Q^{k}_{h}(z^k_h)-\phi^{\top}(z^k_h)(\theta+\Psi_{h}\widehat V^{k}_{h+1}))\right]\leq\\
        & \leq\sum_{k=1}^{K}\mathbb{E}_{\pi}\left[\sum_{h=1}^{H} -\phi^{\top}(z^k_h)\left(\theta-\widehat\theta_k-\varepsilon_{k}\right)+2b^{k}_h(z^k_h)+2\epsilon \right]\leq\\
        &\leq \sum_{k=1}^{K}\mathbb{E}_{\pi}\left[\sum_{h=1}^{H} -\phi^{\top}(z^k_h)\left(\theta-\widehat\theta_k-\varepsilon^{(1)}_{k}\right)\right]+\mathbb{E}_{\pi}\left[\sum_{h=1}^{H}\sum_{k=1}^{K} 2b^{k}_h(z^k_h)+2\epsilon\right]
    \end{align*}
     We can further bound the sum of exploration bonuses across all episodes using bounds on the sum of elliptic potentials in \cite{whitehouse2023sublinear}(Lemma 5). We will first bound the sum of square bonuses and use Cauchy-Schwartz to finish:
    \begin{align}\label{eq:help_eq_11}
        \sum_{k=1}^{K} \left(b^{k}_h(z^{k}_h)\right)^2=\sum_{k=1}^{K} \beta^2_t(\delta)\left\|\phi(z^{k}_h)\right\|^2_{\mathrm{Cov}_h^{k}( \lambda)^{-1}}&=\beta^2_t(\delta)\sum_{k=1}^{K}\left\|\phi(z^{k}_h)\right\|^2_{\mathrm{Cov}_h^{k}( \lambda)^{-1}}\leq\\
        &\leq 2\beta^2_t(\delta)\gamma_{K}(\lambda)
    \end{align}
     After applying Cauchy-Schwarz to the eq.(\ref{eq:help_eq_11}) we have:
     \begin{align*}
         \sum_{h=1}^{H}\sum_{k=1}^{K} b^{k}_h(z^k_h)\leq H\beta_t(\delta) \sqrt {2K\gamma_K(\lambda)}
        \end{align*}
     Once again, to remove the expectation operator $\mathbb{E}_{\pi}$ with respect to the episode dynamics we use Azuma-Hoefding to obtain w.p. at least $1-\delta$:
     \begin{align*}
         \left|\sum_{k=1}^{K} b^{k}_{h}(z^{k}_h)-\mathbb{E}_{\pi}\left[\sum_{k=1}^{K}b^{k}_{h}(z^{k}_h)\right]\right|\leq \sqrt{K\log\left(\frac{2}{\delta}\right)}
     \end{align*}
      and hence:
      \begin{align}\label{eq:expected_potetntial_sum}
          \mathbb{E}_{\pi}\left[\sum_{h=1}^{H}\sum_{k=1}^{K} b^{k}_h(z^k_h)\right]\leq \widetilde{\cO}\left( H\beta_t(\delta) \sqrt {2K\gamma_K(\lambda)}\right)
      \end{align}
      To bound $(\ddag)$ it remains to bound:
      \begin{align*}
          &\sum_{k=1}^{K}\left(V^{\pi'}(x^{i}_1)-\widetilde V^{\pi'}_1(x^{i}_{1})+\mathbb{E}_{\pi}\left[\sum_{h=1}^{H} -\phi^{\top}\left(z^{(1),k}_h\right)\left(\theta-\widehat\theta_k-\varepsilon_{k}\right)\right]\right)=\\
          =&\sum_{k=1}^{K}\left(\mathbb{E}_{\pi'}\left[\sum_{h=1}^{H}\phi^{\top}\left(z'^{k}_h\right)\right]\left(\theta-\widehat\theta_k-\varepsilon_{k}\right)-\mathbb{E}_{\pi}\left[\sum_{h=1}^{H}\phi^{\top}\left(z^{k}_h\right)\left(\theta-\widehat\theta_k-\varepsilon_{k}\right)\right]\right)=\\
          =&\sum_{k=1}^{K}\mathbb{E}_{z^{k}_h\sim \pi, z'^{k}_h\sim \pi'}\left[\left(\sum_{h=1}^{H} \phi^{\top}\left(z'^{k}_h\right)-\phi^{\top}\left(z^{k}_h\right)\right)\left(\theta-\widehat\theta_k-\varepsilon_{k}\right)\right]\\
      \end{align*}
      Note that $\overline{\phi}_k=\left(\sum_{h=1}^{H} \phi\left(z'^{k}_h\right)-\phi\left(z^{k}_h\right)\right)$ is exactly the trajectory-difference feature (as defined in eq.(\ref{eq:trajectory_lvl_feature})). By applying the Lemmas.(\ref{Lemma:reward_confidence_bounds_app}),     (\ref{lemma:bounding_estimation_noise_app}) we have  :
      \begin{align}
          &\left|\overline{\phi}^{\top}_i\left(\theta-\widehat\theta_k\right )\right|\leq \|\overline{\phi}_k\|_{\mathbf{W}^{- 1}_k}\beta_r(\delta)\\
            &\left|\overline{\phi}^{\top}_i\varepsilon^{(1)}_k\right|\leq \beta_a(\delta)\|\overline{\phi}_k\|_{\mathbf{W}^{-1}_k}
      \end{align}
      Lastly, to finish, we use the standard bound on the sum of elliptic potentials on the trajectory-difference kernel function $\overline{k}$ \citep{srinivas2010gaussian} and Azuma-Hoefding to pass the inequality through the expectation operator $\mathbb{E}_{\pi, \pi'}[\cdot]$:
      \begin{align}
          \sum_{k=1}^{K}\mathbb{E}_{z^{k}\sim \pi, z'^{k}_h\sim \pi'}&\left[\left(\sum_{h=1}^{H} \phi^{\top}\left(z'^{k}_h\right)-\phi^{\top}\left(z^{k}_h\right)\right)\left(\theta-\widehat\theta_k-\varepsilon_{k}\right)\right]\leq \\
          &\leq 2\max(\beta_a(\delta), \beta_{r}(\delta))\sqrt{2K\overline{\gamma}_K(\tau)}
      \end{align}
      Adding the previous inequalities, we now finally have the  bound $(\ddag)$ as :
      \begin{align}\label{eq:help_eq_12}
          (\ddag)\leq \widetilde{\mathcal{O}}\left( H\beta_t(\delta) \sqrt {2K\gamma_K(\lambda)}+ 2KH\epsilon+2\max(\beta_a(\delta), \beta_{r}(\delta))\sqrt{2K\overline{\gamma}_K(\tau)}\right)
      \end{align}
      We now turn our attention to the term $(\ddag\ddag)$ in eq.(\ref{eq:help_eq_4}). The approach to bounding $(\ddag\ddag)$ is identical to the approach we used for $(\ddag)$.\\
      Namely, although the "worst-case" noise $\varepsilon_k^{\dag}$ replaces the exploration process $\varepsilon_{k}$, note that $\varepsilon_k^{\dag}$ is independent of the current episode trajectory and moreover by definition in eq(\ref{eq:worst_case_noise_app}) satisfies expression in Lemma.(\ref{lemma:bounding_estimation_noise_app}). Thus equations eq.(\ref{eq:expected_potetntial_sum})(invariant of the exploration noise) and eq.(\ref{eq:help_eq_12}) continue to hold for $\varepsilon^{\dag}_k$ and hence:
      \begin{align*}
          (\ddag\ddag)\leq \widetilde{\mathcal{O}}\left( H\beta_t(\delta) \sqrt {2K\gamma_K(\lambda)}+ 2KH\epsilon+2\max(\beta_a(\delta), \beta_{r}(\delta))\sqrt{2K\overline{\gamma}_K(\tau)}\right)
      \end{align*}
      By simply adding $(\ddag), (\ddag\ddag)$ we have:
      \begin{align*}
          (\dag)\leq \widetilde{\mathcal{O}}\left( H\beta_t(\delta) \sqrt {2K\gamma_K(\lambda)}+ 2KH\epsilon+2\max(\beta_a(\delta), \beta_{r}(\delta))\sqrt{2K\overline{\gamma}_K(\tau)}\right)
      \end{align*}
      \textbf{Step 2.} Obtaining the final regret bound and tunning $\lambda,\tau, \varepsilon$.\\

      As noted in beginning  of \textbf{Step 1.} second summand of eq.(\ref{eq:regret_1}) ,$(\dag\dag)$ is equal to $(\ddag)$. Thus we can apply the bound derived in eq.(\ref{eq:help_eq_12}) to $(\dag\dag)$ and finally obtain w.p. at least $1-\delta$:
      \begin{align}\label{eq:regret_bound_1}
          &\sum_{k=1}^{K}V_1^\star(x_1^k)-V_1^{\pi}(x_1^k)\leq \nonumber\\
          &\leq\widetilde{\mathcal{O}}\left( H\beta_t(\delta) \sqrt {2K\gamma_K(\lambda)}+ 2KH\epsilon+2\max(\beta_a(\delta), \beta_{r}(\delta))\sqrt{2K\overline{\gamma}_K(\tau)}\right)
      \end{align}
      where the dependency of $\delta$ in the bound is poly-logarithmic. Before finalizing the regret bound, we tune the expression in eq.(\ref{eq:regret_bound_1}) with respect to $\lambda, \tau, \epsilon$.  We first plug in the expression derived for $\beta_r, \beta_{t},\beta_a$ in Lemmas.(\ref{Lemma:reward_confidence_bounds_app},\ref{lemma:conf_bound_trans_lemma_app},\ref{lemma:bounding_estimation_noise_app}) in eq.(\ref{eq:regret_bound_1}) to obtain:
      \begin{align}\label{eq:help_eq_13}
          &\left(H^2\kappa_{\cZ}\sqrt{\left(\frac{\overline{\gamma}_K(\tau)}{\tau}\lor1\right)}\sqrt{\left(\left(\frac{\overline{\gamma}_K(\tau)}{\tau}\lor1\right)\frac{1}{\epsilon^2}\right)^{\frac{1}{\beta_p-1}}+\left(\frac{1}{\epsilon^2}\right)^{\frac{2}{\beta_p-1}}+\gamma_K(\lambda)}+\sqrt{\lambda}H+\epsilon\sqrt{K}\right)\cdot\nonumber\\
          &\cdot \sqrt{K\gamma_K(\lambda)}+ KH\epsilon+ 3H\kappa_{\cZ}\left(\sqrt{\overline{\gamma}_K(\tau)}+\sqrt{\tau}\right)\sqrt{\overline\gamma_K(\tau)K}
      \end{align}
      Here the equation in (\ref{eq:help_eq_13}) bounds that in eq.(\ref{eq:regret_bound_1}) up to poly-logaritmic factors. All bounds in the following derivation are to be understood asymptotically, i.e., $f(K)>g(K)$ if $\lim_{K\rightarrow \infty} \frac{g(K)}{f(K)}=0$.\\
      
      To bound the information gain terms as functions of $K, \tau, \lambda$ we use the result of \cite{vakili2021information} and Lemma.(\ref{lemma:diff_kernel_gain_bound}). We now finally have the asymptotic upper bound on the expression in eq.(\ref{eq:regret_bound_1}) as:
      \begin{align*}
          &\left(\sqrt{\left(\frac{\left(\frac{K}{\tau}\right)^{\frac{1}{\beta_p}}}{\tau}\lor1\right)}\sqrt{\left(\left(\frac{\left(\frac{K}{\tau}\right)^{\frac{1}{\beta_p}}}{\tau}\lor1\right)\frac{1}{\epsilon^2}\right)^{\frac{1}{\beta_p-1}}+\left(\frac{1}{\epsilon^2}\right)^{\frac{2}{\beta_p-1}}+\left(\frac{K}{\lambda}\right)^{\frac{1}{\beta_p}}}+\sqrt{\lambda}+\epsilon\sqrt{K}\right)\cdot\nonumber\\
          &\cdot \sqrt{K\left(\frac{K}{\lambda}\right)^{\frac{1}{\beta_p}}}+ K\epsilon+ 3\left(\sqrt{\left(\frac{K}{\tau}\right)^{\frac{1}{\beta_p}}}+\sqrt{\tau}\right)\sqrt{\left(\frac{K}{\tau}\right)^{\frac{1}{\beta_p}}K}
      \end{align*}
      Before moving on to the optimization, we will impose an assumption that $\tau>\left(\frac{K}{\tau}\right)^{\frac{1}{\beta_p}}$, solve the global optimization problem, and later confirm this. To tune the expression we paramtrize it with respect to the regularizers and the covering mesh as $\lambda=K^{a_1},\tau=K^{a_2},\epsilon=K^{-b}$, whence  we obtain the following piece-wise optimization problem:
      \begin{align*}
          \min_{a_1,a_2, b\in \mathbb{R}^{\geq 0}} &\sqrt{K^{1+\frac{1-a_1}{\beta_p}}\max \left(K^{\frac{4b}{\beta_p-1}}, K^{a_1}, K^{\frac{1-a_1}{\beta_p}}\right)}+ \\
          &+K^{1+\frac{1-a_1}{\beta_p}-b}+\sqrt{K^{1+\frac{1-a_2}{\beta_p}}\max\left(K^{\frac{1-a_2}{\beta_p}}, K^{a_2}\right)}
      \end{align*}

      Note that as $K$ is big enough, this effectively becomes a linear optimization problem:
      \begin{align}\label{eq:help_eq_14}
          \min_{a_1,a_2, b\in \mathbb{R}^{\geq 0}} & 0.5\left(1+\frac{1-a_1}{\beta_p}+\max \left(\frac{4b}{\beta_p-1}, a_1, \frac{1-a_1}{\beta_p}\right)\right)\lor \nonumber\\
          &\lor\left(1+\frac{1-a_1}{\beta_p}-b\right)\lor 0.5\left(1+\frac{1-a_2}{\beta_p}+\max\left(\frac{1-a_2}{\beta_p}, a_2\right)\right)
      \end{align}
      which, while tedious, is solvable by hand. Specifically for $\lambda, \tau=\widetilde\cO\left(K^{\frac{2}{\beta_p+1}}\right), \epsilon=\widetilde{\cO}\left(K^{-\frac{\beta_p-1}{2(\beta_p+1)}}\right)$ we have $\tau\geq \widetilde{\cO}(\overline{\gamma}_K(\tau)),\sqrt{\lambda} \sim \beta_t$ satisfying the optimization condition and the confidence width upper bound in Lemma. (\ref{lemma:conf_bound_trans_lemma_app}) respectively. Furthermore, the regret bound in eq.(\ref{eq:regret_bound_1}) now becomes:
      \begin{align*}
          \sum_{k=1}^{K}V_1^\star(x_1^k)-V_1^{\pi}(x_1^k)\leq \widetilde{\cO}\left(K^{1-\frac{(\beta_p-1)^2}{2\beta_p(\beta_p+1)}}\right)=o(K)
      \end{align*}
      which is indeed sublinear in $K$.
      
      Further, note that the only difference in policies  $\pi,\pi'$ is the realization of (i.i.d) exploration noise $\left\{\varepsilon'_k\right\}_{k=1}^{K}$ and thus the bound in eq.(\ref{eq:regret_bound_1}) applies to policy $\pi'$ as well. We finally also choose $\lambda, \tau, \epsilon$ as  functions of the episode length $\kappa_{\cZ}, H$  as :
      \begin{align*}
        & \tau= H^2\log(K)^{\beta_p}K^{\frac{2}{\beta_p+1}}\\
          &\lambda= \log(K)^{\beta_p}K^{\frac{2}{\beta_p+1}}\\
          &\epsilon=H^{2}\kappa_{\cZ}K^{-\frac{\beta_p-1}{2(\beta_p+1)}}
      \end{align*}
      so that $\frac{\overline{\gamma}_K(\tau)}{\tau}$, $\frac{\beta_{t}(\delta)}{\sqrt{\lambda}}$ can be optimized
      as functions of constants of $H,\kappa_{\cZ}$.
      Thus, finally, w.p. $1- \delta$ we have:
      \begin{align*}
          \mathfrak{R}(K)\leq \widetilde{\cO}\left(\kappa_{\cZ}\max\left(H^{2.5}K^{1-\frac{(\beta_p-1)^2}{2\beta_p(\beta_p+1)}},H^{3}K^{1-\frac{\beta_p-1}{2(\beta_p+1)}}\right)\right)
      \end{align*}

    \end{proof}
\newpage
\section{Appendix B.}\label{sec:app_B}
In this section, we derive some commonly seen quantities in GP optimization on the case of the trajectory-difference kernel $\overline{k}$

An important quantity in establishing tractable regret bounds in GP algorithms is \textit{information gain}\citep{vakili2021information}. For a set of points $\cZ_T=\{z_1, z_2,\dots z_T\}$ the information gain is defined as $\gamma_{\cZ_T}(\rho)=\max_{\cZ_T}\log \det(\bfI_{T}+\rho^{-1}\bfK_{\cZ_T})$ along with the worst case information gain $\gamma_T(\rho)=\mathrm{sup}_{\cZ_T}\gamma_{\cZ_T}(\rho)$. Information gain is provably sub-linear in $T$, and its asymptotic bounds are derived in \cite{vakili2021information}. Specifically relevant to our work are the information-gain bounds for the \textit{trajectory-difference} kernel $\overline{\gamma}_K(\rho)$, which we derive here.

\begin{definition}(Trajectory difference kernel $\overline{k}$)
     Consider the trajectory-difference feature function $\overline{\phi}:\mathbb{R}^{2dH}\rightarrow \cH_{k}$, mapping ordered trajectory pairs of state-action samples to the  RKHS of the kernel $k$, $\cH_k$:
\begin{align*}
    \overline{\phi}(\tau, \tau')=\sum_{z\in \tau, z'\in \tau'} \phi(z)-\phi(z')
\end{align*}
 and  let $\overline{k}((\tau_1,\tau'_1),(\tau_2,\tau'_2))=\langle\overline{\phi}(\tau_1, \tau_1'), \overline{\phi}(\tau_2, \tau_2')\rangle_{\cH_k}$ be the the kernel function induced by the said features.\footnote{The inner product is inherited from $\cH_k$, clearly $\cH_{\overline{k}}\subseteq{\cH_k}$}
\end{definition}
\begin{lemma}\label{lemma:diff_kernel_gain_bound}(bounding information gain of the trajectory-difference features) Consider the  the information gain defined as in \cite{vakili2021information} for the case of the trajectory-difference kernel. The following bound on the information gain holds:
\begin{align*}
    \overline{\gamma}_t(\rho)=&\max_{\{(\tau_i,\tau'_i)\}^{t}_{i=1}} \; \log \det\left(\rho^{-1}\sum_{i=1}^{t}\overline{\phi}(\tau_i,\tau'_i)\otimes \overline{\phi}(\tau_i,\tau'_i)+\mathbf{Id}\right)\leq\\
    &\leq  \left(\frac{2C_pF^2tH^2 }{\rho}\right)^{\frac{1}{\beta_p}}\log^{-\frac{1}{\beta_p}}\left(1+\frac{2tH^2}{\rho}\right)\log{K}
\end{align*}
Where $\det(\cdot)$ is the operator, Fredholm determinant \citep{SVM_Book}. In particular we have $\overline{\gamma
}_t(\rho)\leq \widetilde{\cO}\left(\left(\frac{tH^2}{\rho}\right)^{\frac{1}{\beta_p}}\right)$

\end{lemma}
\begin{proof}
Let and $\tau_i=\left(z^{i}_1,z^{i}_2,\dots z^{i}_H\right)$ and $\tau'_i=\left(z'^{i}_1,z'^{i}_2,\dots z'^{i}_H\right)$. We can expand the expression within the $\det(\cdot)$  as :
\begin{align*}
    &\sum_{i=1}^{t}\overline{\phi}(\tau_i,\tau'_i)\otimes \overline{\phi}(\tau_i,\tau'_i)=\sum_{i=1}^{t}\sum_{z,z'\in \tau_i\times\tau'_i} (-1)^{\overline{(z,z')\in \tau_i \lor (z,z')\in \tau'_i}}\phi(z)\otimes\phi(z') 
\end{align*}
Note that $\pm\left(\phi(z)\otimes\phi(z')+\phi(z)\otimes\phi(z')\right)\preceq \phi(z)\otimes\phi(z)+\phi(z')\otimes\phi(z')$ by "completing the square" and thus we have:
\begin{align*}
    &\sum_{i=1}^{t}\overline{\phi}(\tau_i,\tau'_i)\otimes \overline{\phi}(\tau_i,\tau'_i)\preceq 2H\sum_{i=1}^{t}\sum_{h=1}^{H} \left(\phi(z^{i}_h)\otimes\phi(z^{i}_h)+\phi(z'^{i}_h)\otimes\phi(z'^{i}_h)\right)
\end{align*}
    Note that the left side of the above equation is a symmetric operator and thus $\log$ is monotonically increasing over operators we have:
    \begin{align*}
        &\log \det\left(\rho^{-1}\sum_{i=1}^{t}\overline{\phi}(\tau_i,\tau'_i)\otimes \overline{\phi}(\tau_i,\tau'_i)+\mathbf{Id}\right)\leq\\
        &\leq \log\det\left(\frac{2H}{\rho}\sum_{i=1}^{t}\sum_{h=1}^{H} \left(\phi(z^{i}_h)\otimes\phi(z^{i}_h)+\phi(z'^{i}_h)\otimes\phi(z'^{i}_h)\right)+\mathbf{Id}\right)\leq\\
        &\leq  \max_{\{x_i\}_{i=1}^{2tH}} \log\det\left(\frac{2H}{\rho}\sum_{i=1}^{2tH}\phi(x_i)\otimes\phi(x_ i)+\mathbf{Id}\right)= \gamma_{2tH}\left(\frac{\rho}{H}\right)
    \end{align*}
    Where $\gamma_t(\cdot)$ is the information gain defined wrt the original kernel $k$. We can now directly use the information gain bound in \cite{vakili2021information} to finish:
    \begin{align*}
        \overline{\gamma}_t(\rho)\leq \left(\frac{2C_pF^2tH^2 }{\rho}\right)^{\frac{1}{\beta_p}}\log^{-\frac{1}{\beta_p}}\left(1+\frac{2tH^2}{\rho}\right)\log{K}
    \end{align*}
\end{proof}

\newpage

\section{Appendix C}\label{sec:discuss_appendix}

In this section, we add a more technical discussion around the difficulty of adapting the ideas from the works that consider linear trajectory-feedback setting and action-level feedback in kernel MDPs.
\paragraph{Difficulty of extending current approaches to kernel RL preference feedback.}\label{sec:difficulty_kernel_rl_app} 
Extending preference-based RL to kernel MDPs introduces technical obstacles absent in both linear preference-based RL and kernel RL with action-level rewards. The central challenge is controlling the log-covering number of the value-function class, which directly determines the widths of confidence intervals and is essential for optimism-based analyses~\citep{jin2020provably}.

In linear MDPs, the log-covering number is bounded polynomially in the ambient dimension, does not grow with the number of episodes, and only exhibits logarithmic growth with respect to the covering parameter $\epsilon$~\citep{jin2020provably}. Consequently, the covering mesh $\epsilon$ can be set to $1/\text{poly}(K)$ while the log-covering number remains $\widetilde{\cO}(1)$, leaving the regret order unchanged. In kernel settings, the situation is fundamentally different. Recall that the parameter $\beta_p = 2\nu/d + 1$ captures the smoothness of the kernel relative to the input dimension; the log-covering number of the value function class $\cV$  satisfies:~\citep{Vakili23-KerRL, yang2020function}:
\begin{align}\label{eq:appendix_cover}
    &\cV=\left\{\textsc{clip}\left(\max_{a\in \cA}\phi^{\top}(z)\theta+\beta\sqrt{\phi^{\top}(z)\bfB \phi(z)}, C_{\text{clip}}\right), \|\theta\|_{\cH_k}\leq R_{\text{norm}}, \|\bfB\|_2\leq \lambda^{-1}, \beta\leq B_{\text{optim}}\right\}\nonumber\\
    &\log \cN_{\infty}(\cV, \epsilon)= \cO\left(\left(\frac{R^{2}_{\text{norm}}}{\epsilon^2}\right)^{\frac{1}{\beta_p-1}}+\left(\frac{B_{\text{optim}}^2}{\lambda\epsilon^2}\right)^{\frac{2}{\beta_p-1}}\right)
\end{align}
The critical difficulty is that the exponent $1/(\beta_p - 1)$ diverges as $\beta_p \to 1$ (i.e., when $d \gg 2\nu$). Since the estimator norm $R_{\text{norm}}$, the confidence width $B_{\text{optim}}$, and the covering mesh $\epsilon$ are all polynomial in $K$, this means the log-covering number itself can grow as an arbitrary-degree polynomial in $K$, rendering standard regret analyses vacuous. It is almost universally the case that the hardest modes of operation for kernel algorithms(both bandit and MDPs) are exactly when $\beta_p\sim 1$. Frequently, algorithms fail to achieve sub-linearity in the "left tail" of the smoothness parameter $\beta_p$ \cite{chowdhury2019online,yang2020function,maran2024no, srinivas2010gaussian}\\
Kernel preference feedback adds an additional layer of difficulty to bounding the log covering numbers; exploration necessitates the addition of a linear exploration bonus in the form of a GP(see eq.(\ref{eq:noise_definition})). This exploration noise is not sufficiently smooth and lies outside of the $\nu$-Mat\'ern RKHS w.p. 1 \citep{kanagawa1807gaussian}, making the bounds in eq.(\ref{eq:main_text_cover}) inapplicable. While sampling-based exploration in RL is certainly not new \cite{chowdhury2017kernelized}, value iteration poses new challenges for Thompson sampling. While  GP-TS \cite{chowdhury2019online} can simply disregard the complexity of the estimated/sampled posterior and only concentrate on the $\ell_{\infty}$ bound, value iteration requires regularity of the estimators, implicitly through the covering bounds.

\paragraph{Comparison to the work of \cite{Vakili23-KerRL}}The current state of the art for action-level kernel MDPs \citep{maran2024local} considers space partitioning in which the Bellman operator admits a low-error, locally linear approximation and avoids value iteration over a nonparametric functional family altogether. While \cite{maran2024local} achieves an impressive no-regret performance for a wide family of smooth MDPs, domain partitioning is fundamentally unapplicable in trajectory-level feedback.  The stochasticity of transitions prevents advance determination of which trajectory (and thus which part of the space) will be realized, precluding localized Bellman operator estimation. In addition, standard elliptic exploration bonuses cannot be computed causally; they require knowing the whole trajectory before the episode begins. The work of \citep{Vakili23-KerRL} recently claimed a very impressive \textit{order-optimal} regret under action-level feedback; however, \cite{maran2024local} pointed out a subtle flaw in the use of domain partitioning in conjunction with value iteration. Namely, space localization methods aim to obtain tighter confidence bounds by constructing estimators that use only samples from a small subset of the domain. The issue with this method in value iteration is that while we can control the state-action pair $(x^{h}_i, a^h_i)$, there is no way to localize the next transition state $x^{h+1}_i$. This is crucial in value-iteration, as $V(x^{h+1}_i)$ is used as an estimate for $\Psi_{h}V(x^{h}_i, a^h_i)$. Our work avoids this issue by using uniform covering bounds; that is, we do not localize the value function to any part of the domain. As a result, our covering bounds are much less optimistic than those in \citep{Vakili23-KerRL} and require careful tuning of the cover parameter $\epsilon$ to obtain a no-regret performance\\

An additional research direction relevant to preference learning with kernels is trajectory-level feedback learning in linear MDPs. Causality of the exploration bonus in linear MDP's is adressed by sampling Gaussian noise from the trajectory covariance~\citep{shani2020optimistic, cassel2024near}. Exploration in linear MDPs requires exploration only along $d$ directions, hence the addition of the exploratory noise does not alter the order of the log covering bounds in \cite{jin2020provably}.

\paragraph{Our technical contributions.}
The key to obtaining tight confidence bounds and meaningful regret guarantees is controlling the RKHS norm of the $Q$-function estimator. Our contribution addresses this through two main techniques applied to (i) the kernel logistic ridge regression (KLRR) reward estimator (eq.~\ref{eq:reward_theta_calc}) and (ii) the exploration noise (eq.~\ref{eq:noise_definition}).

\textbf{Tuning the reward regularizer $\tau$.} 
The central idea for controlling the KLRR estimator's RKHS norm is careful tuning of the regularizer $\tau$. Increasing $\tau$ places more weight on the norm penalty, yielding smaller-norm estimators, but creates a bias-variance trade-off. Critically, $\tau$ also controls the RKHS norm of the exploration process: as shown in \citet{whitehouse2023sublinear}, tuning $\tau$ enables tighter bounds on the confidence width $\beta_{\text{reward}}(\delta)$, allowing us to reduce exploration variance (and thus RKHS norm) when the estimator is more certain. However, the relationship between $\beta_{\text{reward}}(\delta)$ and $\tau$ is nonlinear; increasing $\tau$ excessively can worsen the confidence width of the KLRR estimator.

\textbf{Covering bounds for the posterior GP $\varepsilon_k$.}
As shown in \cite{kanagawa1807gaussian}, $\varepsilon_k$ does not possess sufficient smoothness to be in $\nu$-Mat\'ern RKHS and instead is a member of a more coarse RKHS, $\nu'=\nu-d/2$ Mat\'ern reproducing Hilbert space. As a result, the usual covering bounds in \cite{yang2020function} are inapplicable, and moreover, for $\beta_p<2$, the sampled GP does not belong to any Mat\'ern RKHS, thus necessitating an entirely different approach. To this end, we use the result of \cite{van2008rates}(Theorem 2.1) linking the log-covering bounds of the GP(with covariance $k(\cdot, \cdot)$)  sample paths and the probability that the GP stays within a($\ell_\infty$) ball(small ball probability). Intuitively, while the GP sample paths lie outside of any $\nu$-Mat\'ern RKHS
ball, with high probability they have a neighbor inside of a (reasonably sized) RKHS ball that is only $\epsilon$, away in $\ell_{\infty}$. Hence $\ell_{\infty}$, $\epsilon$-thickened RKHS ball contains almost all GP sample paths and has small metric entropy. With machinery from \cite{van2008rates}, we modify their result for the posterior GP $\varepsilon_k$. The posterior GP $\varepsilon_k$ is \textit{less random} than its prior $\mathbf{GP}(0, k)$ and thus we expect that the small ball probability for $\varepsilon_k$ is at least that of $\mathbf{GP}(0,k)$. This proves to be the case, and in a satisfying way, large $\tau$ necessitates less exploration noise while large $\beta_{\text{reward}}$ increases the noise. Indeed the log-covering bounds for the sample paths of $\varepsilon_k$ can be bounded (with high probability) as
\begin{align*}
   &\varepsilon_k\in G_{k, \epsilon,\delta} \quad \text{w.p.}\quad 1-\delta\\
   &\log \mathcal N_\infty(\mathcal G_{k, \epsilon,\delta},\epsilon)
    \le\widetilde{\mathcal O}\left(\left(\frac{\beta_r}{\sqrt{\tau}\epsilon}
        \right)^{\frac{2}{(\beta_p-1)}}+\log(1/\delta)\right).
\end{align*}

\textbf{Tuning the transition regularizer $
\lambda$ and the covering parameter $\epsilon$.}
The transition regularizer $\lambda$ governs a bias-variance tradeoff 
analogous to $\tau$. The elliptic exploration bonus is defined through a norm induced by a $\lambda \mathbf{Id}$-shifted covariance operator, hence altering $\lambda$ directly alters the local behavior of the elliptic bonus and hence the covering bounds. A small $\lambda$ 
can provide highly confident(transition) estimator but inflates the covering number of the bonus; 
a large $\lambda$ reduces the covering number at the cost of overly 
pessimistic confidence intervals.
The covering mesh $\epsilon$ introduces a tradeoff that is unique to the kernel regime. In linear MDPs, the log covering number depends only polylogarithmically on $1/\epsilon$, so $\epsilon$ can be tuned freely without polynomial cost. In kernel MDPs, however, eq.~\eqref{eq:main_text_cover} shows that this dependence is polynomial in $1/\epsilon$, and can be of 
arbitrarily high degree in hard kernel instances ($\beta_p\to 1$). 
The two sides of the $\epsilon$ tradeoff are as follows. We run kernel regression on the nearest neighbor of $\widehat{V}^{k}_{h+1}$ in the 
$\ell_\infty$ cover rather than on $\widehat{V}^{k}_{h+1}$ itself, which 
introduces a misspecification error in learning the next-state value 
function $\Psi_h\widehat{V}^{k}_{h+1}$. This causes $\beta_{\text{trans}}$ 
to scale as $\widetilde{\mathcal{O}}(\epsilon\sqrt{K})$,%
\footnote{See eq.~\eqref{eq:transition_conf_width} for details.} 
so a large $\epsilon$ leads to trivially wide confidence intervals. 
Conversely, a small $\epsilon$ controls misspecification but implies
large bounds on the covering number. We tune $\epsilon$ to balance these two effects.

\newpage

\section{Appendix D}\label{sec:exp_app}

In this section, we run a simple simulation emprirically confirming the theoretically proven upper bounds in Theorem \ref{Thrm:main_theorem} \\

We run \textsc{PROSTO} on the synthetic MDP with $\cS=[0,1]^2,\cA=[0,1]$. The transition function is parametrized as a spherical Gaussian function:
\begin{align}
    P_h(x'|x,a)\sim\exp\left(-\frac{\|x'-[x\; a]^{\top}\mathbf{A}\|^2_2}{2}\right)\quad \mathbf{A}=\begin{bmatrix}
    0.5& 0 & 0.5\\
    0& 0.5 & 0.5\\
   \end{bmatrix}^{\top}\label{eq:MDP_2}
\end{align}
We plot the cumulative and average regret of \textsc{PROSTO} for frequently used functions in Bayesian optimization: Hartman \cite{picheny2013benchmark} and Ackley \cite{eriksson2019scalable}. We parametrize both as a function state-action $\mathbf{z}=(s,a)$ variable below:
\begin{align}
 &r_{\text{hartman}}(\mathbf{z}) = -\sum_{i=1}^{4} \alpha_i \exp\!\left(-\sum_{j=1}^{3} A_{ij}(z_j - P_{ij})^2\right), \qquad \mathbf{z}\in[0,1]^3\label{eq:hartman}\\
 &\boldsymbol{\alpha} = (1,\;1.2,\;3,\;3.2)^\top\quad A = \begin{pmatrix} 3 & 10 & 30 \\ 0.1 & 10 & 35 \\ 3 & 10 & 30 \\ 0.1 & 10 & 35 \end{pmatrix}, \qquad P = 10^{-4}\begin{pmatrix} 3689 & 1170 & 2673 \\ 4699 & 4387 & 7470 \\ 1091 & 8732 & 5547 \\ 381 & 5743 & 8828 \end{pmatrix}\nonumber\\
 & r_{\text{ackley}}(\mathbf{z}) = -20\exp\!\left(-0.2\sqrt{\frac{1}{3}\sum_{i=1}^{3}z_i^2}\right) - \exp\!\left(\frac{1}{3}\sum_{i=1}^{3}\cos(2\pi z_i)\right) + 20 + \exp(1) \label{eq:ackley}
\end{align}

Both functions are scaled and translated so that their image is $[0,1]$. Note that both functions are non-convex and admit several local minima over the optimization domain.
\begin{figure}[htbp]
    \centering

    \begin{subfigure}{0.48\linewidth}
        \centering
        \includegraphics[width=\linewidth]{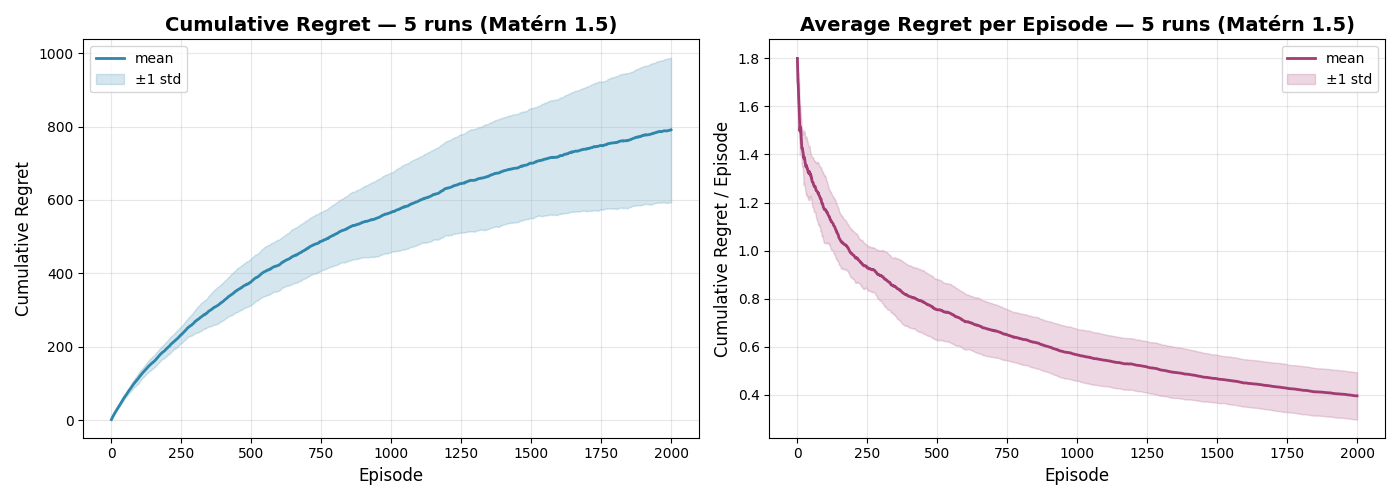}
        \caption{Hartman reward function, Mat\'ern $1.5$ kernel}\label{img:hartman_1}
    \end{subfigure}
    \hfill
    \begin{subfigure}{0.48\linewidth}
        \centering
        \includegraphics[width=\linewidth]{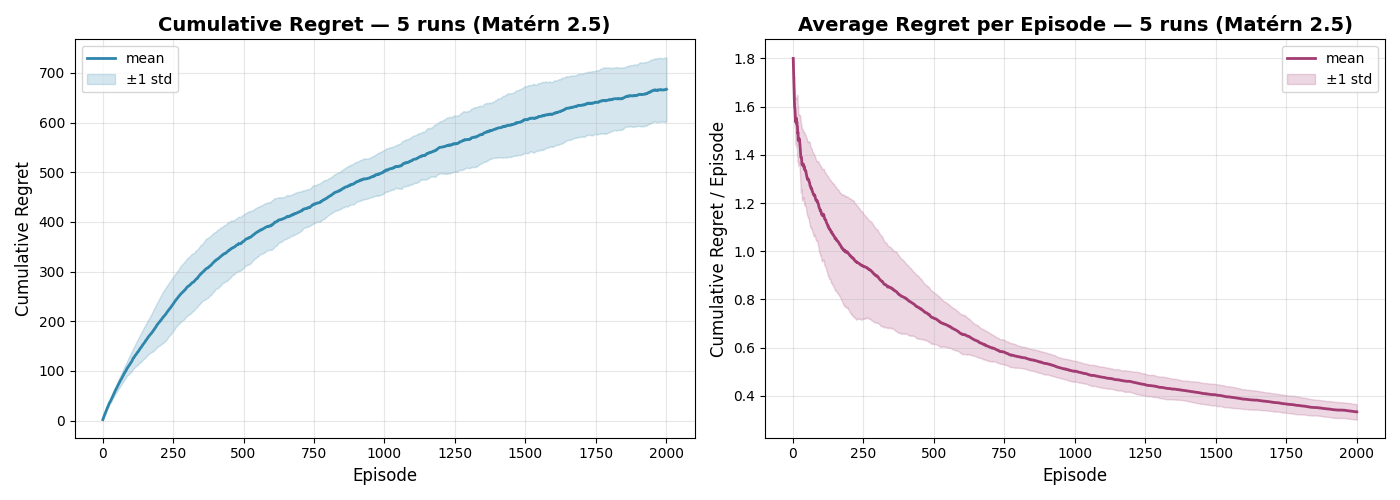}
        \caption{Hartman reward function, Mat\'ern $2.5$ kernel}\label{img:hartman_2}
    \end{subfigure}

    \vspace{0.3cm}

    \begin{subfigure}{0.48\linewidth}
        \centering
        \includegraphics[width=\linewidth]{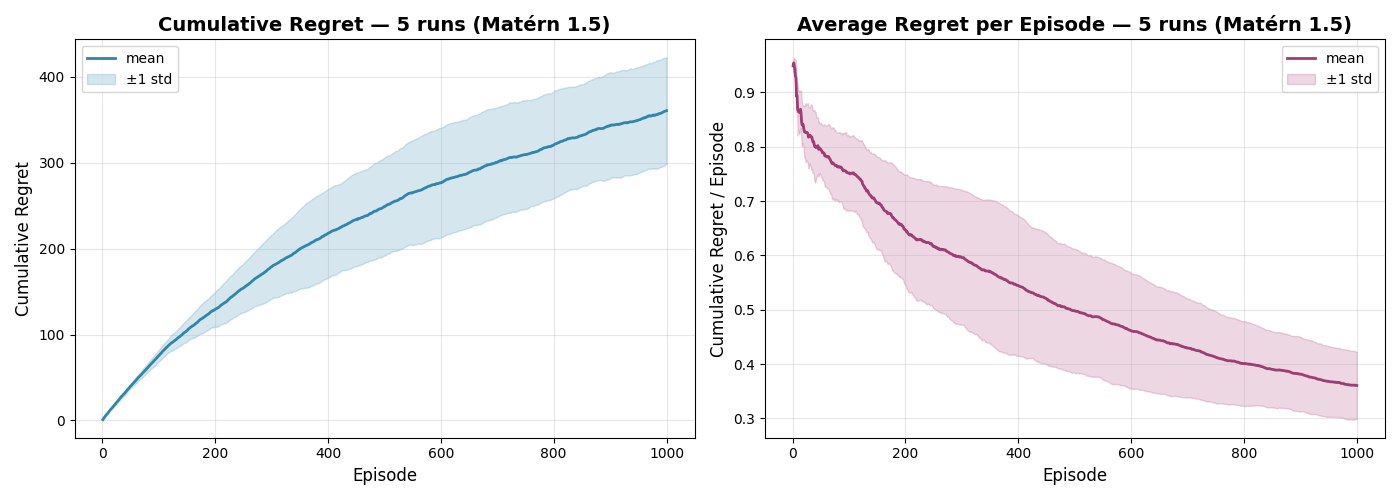}
        \caption{Ackley reward function, Mat\'ern $1.5$ kernel}
    \end{subfigure}
    \hfill
    \begin{subfigure}{0.48\linewidth}
        \centering
        \includegraphics[width=\linewidth]
        {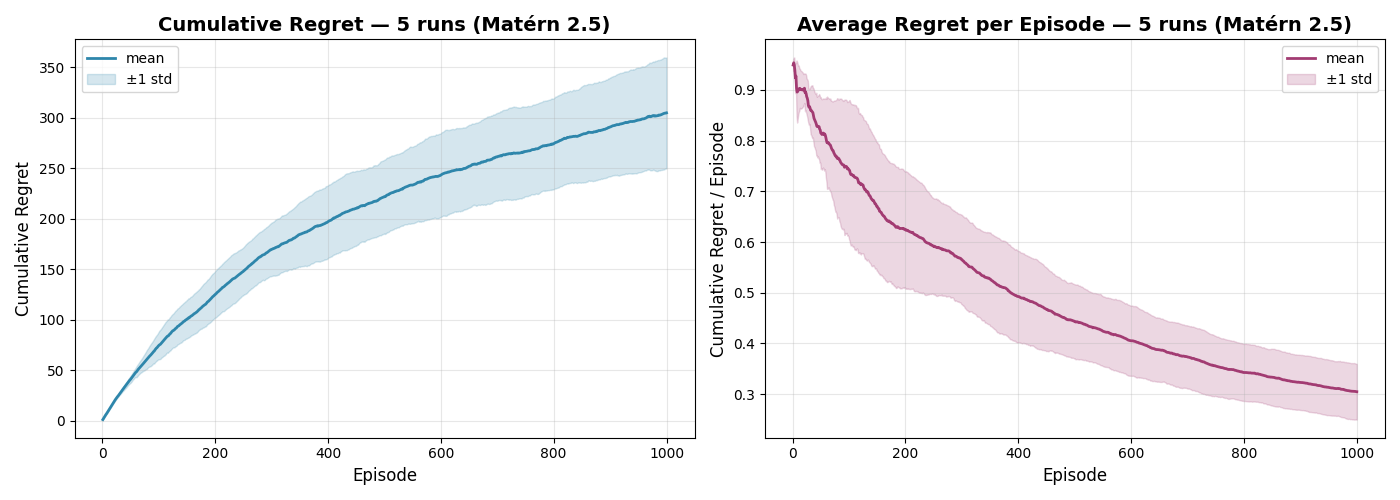}
        \caption{Ackley reward function, Mat\'ern $2.5$ kernel}
    \end{subfigure}

    \caption{Cumulative and average regret for an MDP with Hartman (\ref{eq:hartman})(top row) and Ackley (\ref{eq:ackley})(bottom row) reward functions.}
    \label{fig:hartman_regret}
\end{figure}
We plot the cumulative and average regret of $\textsc{PROSTO}$, where the average regret is defined as $\mathfrak{R}_{\text{avg}}(k)=\frac{\mathfrak{R}(k)}{k}$. From the above picture, we can see $\mathfrak{R}_{\text{avg}}(k)$ decays to zero as the episode count $k$ increases, thus demonstrating a sub-linear regret guarantee.\\
We also plot the log dependency $\log{\mathfrak{R}}(\log k)$ of the cumulative regret as a function of log episode count. This graph offers an interesting comparison between the empirical log regret and the (now linear) upper bound. We plot the log dependency for Hartman and Branin \cite{picheny2013benchmark} reward function

\begin{figure}[htbp]
    \centering

    \begin{subfigure}{0.48\linewidth}
        \centering
        \includegraphics[width=\linewidth]{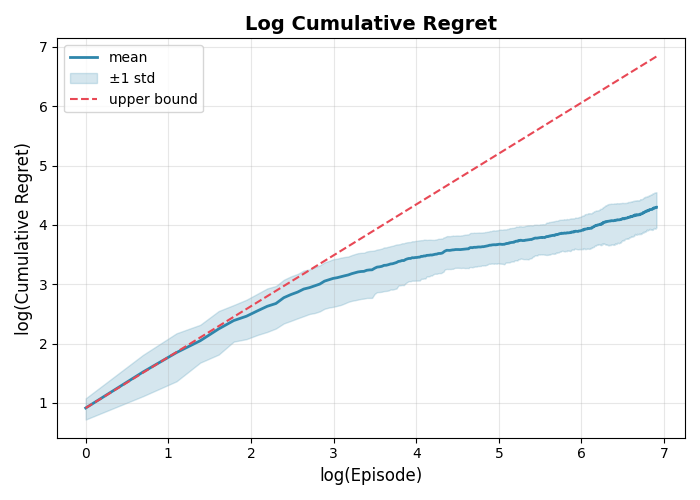}
        \caption{Mat\'ern $2.5$ kernel}
    \end{subfigure}
    \hfill
    \begin{subfigure}{0.48\linewidth}
        \centering
        \includegraphics[width=\linewidth]{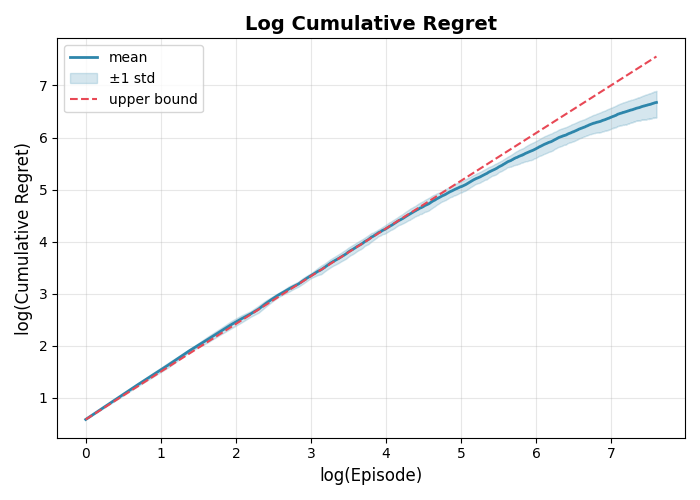}
        \caption{Mat\'ern $1.5$ kernel}
    \end{subfigure}
    \caption{Log regret dependency for Branin(left) and Hartman(right) reward functions}
\end{figure}

The log upper bound is calculated as:
\begin{align*}
    \text{upper}(\log k)=\left(1-\frac{(\beta_p-1)^2}{2\beta_p(\beta_p+1)}\right)\log{k}
\end{align*}
where the kernel smoothness parameter can be calculated as described in the Sec.(\ref{sec:problem}). As seen in the figure, the predicted upper bound lies above the empirical log regret.

\subsection{Discussion on computational complexity of \textsc{PROSTO}}\label{subsec:comp_complexity}

\paragraph{Computational efficiency.}
We now discuss the computational cost of PROSTO.  The two computationally intesive operations in our algorithm are: \textbf{(i)} Sample the posterior GP noise eq.(\ref{eq:noise_definition}) and \textbf{(ii)} greedily optimizing the surrogate $Q$ function $\widehat{Q}^{k}_{h}$ over the entire action set. The remaining operations are logistic regression and $\ell_2$-norm kernel ridge regression, both of which are the result of a convex optimization objective and can thus be efficiently computed over the number of episodes $K$. While various speedups, such as RLS \cite{calandriello2019gaussian}, infrequent Gramian calculation, momentum gradient methods etc.,  are possible for both, here we focus on \textbf{(i), (ii)}.

\paragraph{Calculating (i).}

While $\varepsilon_k$ is not necessarily a member of a $\nu$-Mat\'ern RKHS by a similar argument to that in Lemma.(\ref{lemma:unifom_bound_function}) applied do a slightly differently constructed feature $\phi(z,z')=\phi(z)-\phi(z')$ we have:
\begin{align*}
    \forall z, z'\in  \cZ\langle \phi(z)-\phi(z'), \varepsilon_k\rangle \leq \widetilde{\cO}\left( C
        \sqrt{
            \frac {2d}{\alpha}
            \log\!\left(
                C L_\phi
            \sqrt{\frac{k+\lambda}{\lambda}}\right)}\|\phi(z)-\phi(z')\|_{\bfW^{-1}_k}
        \right)
\end{align*}
Note that LHS contain an $\log(1/\|z-z'\|_2)$ factor which is collapsed into $\widetilde{\cO}(\cdot)$.
By additionally using smoothness of the covariance induced norm $\|\cdot\|_{\bfW^{-1}_k}$ essentially gives us $\alpha=\min(\nu, 1)$ modulus of continuity for the posterior GP $\varepsilon_k$. Equivalently, we only need to calculate $\varepsilon_k$ in $\cO(K^{2d/\alpha})$(say equally spaced into a grid $\cG$) points from the domain $\cZ$ at the price of $\widetilde{\cO}(H)$ additive regret error terms. \\
After this point, this is a tractable, but computationally intensive task. Essentially we need to generate $\mathbb{R}^{\cO(K^{2d/\alpha})}$ dimensional , unbiased Gaussian from a covariance 
$$\mathrm{Cov}(z_1,z_2)=\beta^2_{\text{reward}}(\delta)\langle\phi(z_1),\phi(z_2)\rangle_{\bf{W}^{-1}_k}, (z_1, z_2)\in \cG\times \cG$$
, which can be done by simply diagonalizing $\{\mathrm{Cov}(z_1,z_2)\}_{(z_1,z_2)\in \cG\times \cG}$. Hence the complexity of $\cO(K^{6d/\alpha})$. We note that other sampling-based methods in GP bandits almost ubiquitously assume a finite discretization of the domain \citep{chowdhury2019online, salgia2023random}. 

\paragraph{Calculating (ii)}
For continuous action spaces, the greedy step can either be implemented using a discretization of the state-action space $\cZ$ or it can be abstracted via an
optimization oracle for $\arg\max_{a\in\mathcal A} \widehat Q_h^k(x,a)$ and ommited from overall complexity considerations
as is standard in kernel bandit and kernel RL analyses \citep{salgia2023random, valko2013finite, yang2020function, li2022gaussian} . Here we develop the former argument in more detail, in fact we show that for Mate\'rn kernels the greedy optimization step can always be found in $\cO(K^{2d/\alpha})$  run time while paying for additional additive $\widetilde{\cO}(H)$ regret. The argument here is similar to \textit{"discrete to  continous"} confidence bounds guarantees  seen often in kernel bandits \citep{vakili2021optimal, li2022gaussian}.\\
In case of \textsc{PROSTO} the surrogate state-action function can be written as:
\begin{align*}
    \widehat{Q}^{h}_{k}(z)= \langle \phi(z), \widehat{r}\rangle_{\cH_k}+ \beta \|\phi(z)\|_{\bfB^{-1}}
\end{align*}
where specifically in Alg.(\ref{alg: algorithm}), $\widehat{r}$ is the sum of KLRR and the exploratory GP bonus, and $\bfB$ is the covariance operator where $\|\bfB\|_2\leq \lambda^{-1}$. While $\widehat{r}$ is $\alpha$ H\"older continuous by the same argument as in \textbf{(i)}. To settle the approximation of the elliptic exploration bonus, we simply use the continuity of the operator $\bfB$ induced norm:
\begin{align*}
    &\left|\|\phi(z_1)\|_{\bfB^{-1}}- \|\phi(z_2)\|_{\bfB^{-1}}\right|\leq\\
    &\leq \|\phi(z_1)- \phi(z_2)\|_{\bfB^{-1}}\leq\lambda^{-1/2}\|(\phi(z_1)-\phi(z_2)\|_{\cH_k}\leq \cO\left(\lambda^{-1/2}\|z_1-z_2\|^{\alpha}_2\right)
\end{align*}

Not let $\cG$ be a mesh grid, with mesh size $K^{-2/\alpha}$ and  let $[z]= \mathrm{argmin}_{g \in \cG}\|z-g\|_2$. By previous argument we have:
\begin{align*}
    \left|\widehat{Q}^{h}_{k}(z)-\widehat{Q}^{h}_{k}([z])\right|\leq \cO\left(\max(\beta\lambda^{-1/2}, \|\widehat{r}\|_{\cH_k})\|z_1-z_2\|^{\min(\nu, 1)}\right)
\end{align*}
while  we have developed stronger bounds for $\beta,\|\widehat{r}\|_{\cH_k}$ here a trivial upper bound or $K$ suffices. Finally we have
\begin{align*}
    \left|\widehat{Q}^{h}_{k}(z)-\widehat{Q}^{h}_{k}([z])\right|\leq \cO\left(\frac{1}{K}\right)
\end{align*}

hence using discrete version of $\widehat{Q}$ incurs at most $\widetilde{\cO}(H)$ additive cost while requiring only polynomial run-time as claimed.

\paragraph{Comparison with computational aspects of prior kernel RL algorithms.}
This computational guarantee should be contrasted with existing no-regret algorithms for
kernel and smooth MDPs.  Algorithms such as KOVI\citep{yang2020function} are computationally efficient in the
usual kernel-LSVI sense, but their regret guarantees are not sublinear for the hard low-smoothness Mat\'ern regimes considered here. On the other hand, the recent smooth-MDP
algorithms that obtain no-regret guarantees at this level of generality rely on optimization oracles for which a polynomial implementation is unclear.  In
particular, \citep{maran2024local} explicitly notes that the corresponding constrained continuous
optimization problems are computationally intractable for Smooth Kernel MDP's
and that a naive discretization leads to a complexity exponential $K$.  In particular, for the worst-case complexity in \citep{maran2024local} for the chosen covering parameter is $\widetilde{O}\left(\exp\left(K^{\frac{d}{(\beta_p+1)}}\right)\right)$.\footnote{As noted in Appendix D in the same paper the complexity of solving the non-convex optimization is $\cO\left(K^{1.5\sum_{h\leq H} N_hd_h}\right)$ where $d_h$ is the dimension of the local embedding and $N_h$ is the total number of embeddings. The complexity shown here is directly obtained by plugging parameters obtain in Corollary 19 of the same paper.}

It is unclear if the approach could possibly be implemented efficiently for less general kernel MDPs.\\ 
A domain-partitioning approach in KRVI \cite{Vakili23-KerRL} has been claimed to
achieve both order-optimal regret and polynomial runtime under state-action feedback. However, as explained in this work and \citep{maran2024local}, there is an issue with the covering argument for the surrogate value function, and it is unclear how it can be patched.
 Consequently, to the best of our knowledge, PROSTO is the first
 algorithm under the preference kernel MDP model that simultaneously achieves
sublinear regret for all kernels in the Mat\'ern family and admits an implementation whose
dependence on the number of episodes $K$ is polynomial.


\newpage
\section*{NeurIPS Paper Checklist}

\begin{enumerate}

\item {\bf Claims}
    \item[] Question: Do the main claims made in the abstract and introduction accurately reflect the paper's contributions and scope?
    \item[] Answer:  \answerYes{} 
    \item[] Justification: All the claims stated in the abstract and introduction are further developed in the main text and proven in the Appendix.
    \item[] Guidelines:
    \begin{itemize}
        \item The answer \answerNA{} means that the abstract and introduction do not include the claims made in the paper.
        \item The abstract and/or introduction should clearly state the claims made, including the contributions made in the paper and important assumptions and limitations. A \answerNo{} or \answerNA{} answer to this question will not be perceived well by the reviewers. 
        \item The claims made should match theoretical and experimental results, and reflect how much the results can be expected to generalize to other settings. 
        \item It is fine to include aspirational goals as motivation as long as it is clear that these goals are not attained by the paper. 
    \end{itemize}

\item {\bf Limitations}
    \item[] Question: Does the paper discuss the limitations of the work performed by the authors?
    \item[] Answer:  \answerYes{} 
    \item[] Justification: The limitations of our work are mentioned in the conclusion section and through out the main result and related works section.
    \item[] Guidelines:
    \begin{itemize}
        \item The answer \answerNA{} means that the paper has no limitation while the answer \answerNo{} means that the paper has limitations, but those are not discussed in the paper. 
        \item The authors are encouraged to create a separate ``Limitations'' section in their paper.
        \item The paper should point out any strong assumptions and how robust the results are to violations of these assumptions (e.g., independence assumptions, noiseless settings, model well-specification, asymptotic approximations only holding locally). The authors should reflect on how these assumptions might be violated in practice and what the implications would be.
        \item The authors should reflect on the scope of the claims made, e.g., if the approach was only tested on a few datasets or with a few runs. In general, empirical results often depend on implicit assumptions, which should be articulated.
        \item The authors should reflect on the factors that influence the performance of the approach. For example, a facial recognition algorithm may perform poorly when image resolution is low or images are taken in low lighting. Or a speech-to-text system might not be used reliably to provide closed captions for online lectures because it fails to handle technical jargon.
        \item The authors should discuss the computational efficiency of the proposed algorithms and how they scale with dataset size.
        \item If applicable, the authors should discuss possible limitations of their approach to address problems of privacy and fairness.
        \item While the authors might fear that complete honesty about limitations might be used by reviewers as grounds for rejection, a worse outcome might be that reviewers discover limitations that aren't acknowledged in the paper. The authors should use their best judgment and recognize that individual actions in favor of transparency play an important role in developing norms that preserve the integrity of the community. Reviewers will be specifically instructed to not penalize honesty concerning limitations.
    \end{itemize}

\item {\bf Theory assumptions and proofs}
    \item[] Question: For each theoretical result, does the paper provide the full set of assumptions and a complete (and correct) proof?
    \item[] Answer:  \answerYes{} 
    \item[] Justification: All theoretical statements are proven in the appendix.
    \item[] Guidelines:
    \begin{itemize}
        \item The answer \answerNA{} means that the paper does not include theoretical results. 
        \item All the theorems, formulas, and proofs in the paper should be numbered and cross-referenced.
        \item All assumptions should be clearly stated or referenced in the statement of any theorems.
        \item The proofs can either appear in the main paper or the supplemental material, but if they appear in the supplemental material, the authors are encouraged to provide a short proof sketch to provide intuition. 
        \item Inversely, any informal proof provided in the core of the paper should be complemented by formal proofs provided in appendix or supplemental material.
        \item Theorems and Lemmas that the proof relies upon should be properly referenced. 
    \end{itemize}

    \item {\bf Experimental result reproducibility}
    \item[] Question: Does the paper fully disclose all the information needed to reproduce the main experimental results of the paper to the extent that it affects the main claims and/or conclusions of the paper (regardless of whether the code and data are provided or not)?
    \item[] Answer:  \answerYes{} 
    \item[] Justification: The experiments are described in detail in Sec.\ref{sec:exp_app}. 
    \item[] Guidelines:
    \begin{itemize}
        \item The answer \answerNA{} means that the paper does not include experiments.
        \item If the paper includes experiments, a \answerNo{} answer to this question will not be perceived well by the reviewers: Making the paper reproducible is important, regardless of whether the code and data are provided or not.
        \item If the contribution is a dataset and\slash or model, the authors should describe the steps taken to make their results reproducible or verifiable. 
        \item Depending on the contribution, reproducibility can be accomplished in various ways. For example, if the contribution is a novel architecture, describing the architecture fully might suffice, or if the contribution is a specific model and empirical evaluation, it may be necessary to either make it possible for others to replicate the model with the same dataset, or provide access to the model. In general. releasing code and data is often one good way to accomplish this, but reproducibility can also be provided via detailed instructions for how to replicate the results, access to a hosted model (e.g., in the case of a large language model), releasing of a model checkpoint, or other means that are appropriate to the research performed.
        \item While NeurIPS does not require releasing code, the conference does require all submissions to provide some reasonable avenue for reproducibility, which may depend on the nature of the contribution. For example
        \begin{enumerate}
            \item If the contribution is primarily a new algorithm, the paper should make it clear how to reproduce that algorithm.
            \item If the contribution is primarily a new model architecture, the paper should describe the architecture clearly and fully.
            \item If the contribution is a new model (e.g., a large language model), then there should either be a way to access this model for reproducing the results or a way to reproduce the model (e.g., with an open-source dataset or instructions for how to construct the dataset).
            \item We recognize that reproducibility may be tricky in some cases, in which case authors are welcome to describe the particular way they provide for reproducibility. In the case of closed-source models, it may be that access to the model is limited in some way (e.g., to registered users), but it should be possible for other researchers to have some path to reproducing or verifying the results.
        \end{enumerate}
    \end{itemize}

\item {\bf Open access to data and code}
    \item[] Question: Does the paper provide open access to the data and code, with sufficient instructions to faithfully reproduce the main experimental results, as described in supplemental material?
    \item[] Answer:  \answerNo{} 
    \item[] Justification: We run synthetic experiments with detailed descriptions of the MDP environment sufficient in its own. 
    \item[] Guidelines:
    \begin{itemize}
        \item The answer \answerNA{} means that paper does not include experiments requiring code.
        \item Please see the NeurIPS code and data submission guidelines (\url{https://neurips.cc/public/guides/CodeSubmissionPolicy}) for more details.
        \item While we encourage the release of code and data, we understand that this might not be possible, so \answerNo{} is an acceptable answer. Papers cannot be rejected simply for not including code, unless this is central to the contribution (e.g., for a new open-source benchmark).
        \item The instructions should contain the exact command and environment needed to run to reproduce the results. See the NeurIPS code and data submission guidelines (\url{https://neurips.cc/public/guides/CodeSubmissionPolicy}) for more details.
        \item The authors should provide instructions on data access and preparation, including how to access the raw data, preprocessed data, intermediate data, and generated data, etc.
        \item The authors should provide scripts to reproduce all experimental results for the new proposed method and baselines. If only a subset of experiments are reproducible, they should state which ones are omitted from the script and why.
        \item At submission time, to preserve anonymity, the authors should release anonymized versions (if applicable).
        \item Providing as much information as possible in supplemental material (appended to the paper) is recommended, but including URLs to data and code is permitted.
    \end{itemize}

\item {\bf Experimental setting/details}
    \item[] Question: Does the paper specify all the training and test details (e.g., data splits, hyperparameters, how they were chosen, type of optimizer) necessary to understand the results?
    \item[] Answer: \answerYes{} 
    \item[] Justification: All details of the experiment are stated in Sec.(\ref{sec:exp_app}).
    \item[] Guidelines:
    \begin{itemize}
        \item The answer \answerNA{} means that the paper does not include experiments.
        \item The experimental setting should be presented in the core of the paper to a level of detail that is necessary to appreciate the results and make sense of them.
        \item The full details can be provided either with the code, in appendix, or as supplemental material.
    \end{itemize}

\item {\bf Experiment statistical significance}
    \item[] Question: Does the paper report error bars suitably and correctly defined or other appropriate information about the statistical significance of the experiments?
    \item[] Answer: \answerYes{} 
    \item[] Justification: All experiments are run 5 times and the one std deviations are shown on the graphs.
    \item[] Guidelines:
    \begin{itemize}
        \item The answer \answerNA{} means that the paper does not include experiments.
        \item The authors should answer \answerYes{} if the results are accompanied by error bars, confidence intervals, or statistical significance tests, at least for the experiments that support the main claims of the paper.
        \item The factors of variability that the error bars are capturing should be clearly stated (for example, train/test split, initialization, random drawing of some parameter, or overall run with given experimental conditions).
        \item The method for calculating the error bars should be explained (closed form formula, call to a library function, bootstrap, etc.)
        \item The assumptions made should be given (e.g., Normally distributed errors).
        \item It should be clear whether the error bar is the standard deviation or the standard error of the mean.
        \item It is OK to report 1-sigma error bars, but one should state it. The authors should preferably report a 2-sigma error bar than state that they have a 96\% CI, if the hypothesis of Normality of errors is not verified.
        \item For asymmetric distributions, the authors should be careful not to show in tables or figures symmetric error bars that would yield results that are out of range (e.g., negative error rates).
        \item If error bars are reported in tables or plots, the authors should explain in the text how they were calculated and reference the corresponding figures or tables in the text.
    \end{itemize}

\item {\bf Experiments compute resources}
    \item[] Question: For each experiment, does the paper provide sufficient information on the computer resources (type of compute workers, memory, time of execution) needed to reproduce the experiments?
    \item[] Answer:\answerYes{}
    \item[] Justification: Our experiments are synthetic simulations and could be run on just  CPU over the course of few hours.
    \item[] Guidelines:
    \begin{itemize}
        \item The answer \answerNA{} means that the paper does not include experiments.
        \item The paper should indicate the type of compute workers CPU or GPU, internal cluster, or cloud provider, including relevant memory and storage.
        \item The paper should provide the amount of compute required for each of the individual experimental runs as well as estimate the total compute. 
        \item The paper should disclose whether the full research project required more compute than the experiments reported in the paper (e.g., preliminary or failed experiments that didn't make it into the paper). 
    \end{itemize}
    
\item {\bf Code of ethics}
    \item[] Question: Does the research conducted in the paper conform, in every respect, with the NeurIPS Code of Ethics \url{https://neurips.cc/public/EthicsGuidelines}?
    \item[] Answer: \answerYes{} 
    \item[] Justification: We conform to ethics guidelines.
    \item[] Guidelines:
    \begin{itemize}
        \item The answer \answerNA{} means that the authors have not reviewed the NeurIPS Code of Ethics.
        \item If the authors answer \answerNo, they should explain the special circumstances that require a deviation from the Code of Ethics.
        \item The authors should make sure to preserve anonymity (e.g., if there is a special consideration due to laws or regulations in their jurisdiction).
    \end{itemize}

\item {\bf Broader impacts}
    \item[] Question: Does the paper discuss both potential positive societal impacts and negative societal impacts of the work performed?
    \item[] Answer: \answerNA{} 
    \item[] Justification: It is a theoretical paper.
    \item[] Guidelines:
    \begin{itemize}
        \item The answer \answerNA{} means that there is no societal impact of the work performed.
        \item If the authors answer \answerNA{} or \answerNo, they should explain why their work has no societal impact or why the paper does not address societal impact.
        \item Examples of negative societal impacts include potential malicious or unintended uses (e.g., disinformation, generating fake profiles, surveillance), fairness considerations (e.g., deployment of technologies that could make decisions that unfairly impact specific groups), privacy considerations, and security considerations.
        \item The conference expects that many papers will be foundational research and not tied to particular applications, let alone deployments. However, if there is a direct path to any negative applications, the authors should point it out. For example, it is legitimate to point out that an improvement in the quality of generative models could be used to generate Deepfakes for disinformation. On the other hand, it is not needed to point out that a generic algorithm for optimizing neural networks could enable people to train models that generate Deepfakes faster.
        \item The authors should consider possible harms that could arise when the technology is being used as intended and functioning correctly, harms that could arise when the technology is being used as intended but gives incorrect results, and harms following from (intentional or unintentional) misuse of the technology.
        \item If there are negative societal impacts, the authors could also discuss possible mitigation strategies (e.g., gated release of models, providing defenses in addition to attacks, mechanisms for monitoring misuse, mechanisms to monitor how a system learns from feedback over time, improving the efficiency and accessibility of ML).
    \end{itemize}
    
\item {\bf Safeguards}
    \item[] Question: Does the paper describe safeguards that have been put in place for responsible release of data or models that have a high risk for misuse (e.g., pre-trained language models, image generators, or scraped datasets)?
    \item[] Answer:\answerNA{} 
    \item[] Justification: 
    \item[] Guidelines:
    \begin{itemize}
        \item The answer \answerNA{} means that the paper poses no such risks.
        \item Released models that have a high risk for misuse or dual-use should be released with necessary safeguards to allow for controlled use of the model, for example by requiring that users adhere to usage guidelines or restrictions to access the model or implementing safety filters. 
        \item Datasets that have been scraped from the Internet could pose safety risks. The authors should describe how they avoided releasing unsafe images.
        \item We recognize that providing effective safeguards is challenging, and many papers do not require this, but we encourage authors to take this into account and make a best faith effort.
    \end{itemize}

\item {\bf Licenses for existing assets}
    \item[] Question: Are the creators or original owners of assets (e.g., code, data, models), used in the paper, properly credited and are the license and terms of use explicitly mentioned and properly respected?
    \item[] Answer: \answerNA{} 
    \item[] Justification: 
    \item[] Guidelines:
    \begin{itemize}
        \item The answer \answerNA{} means that the paper does not use existing assets.
        \item The authors should cite the original paper that produced the code package or dataset.
        \item The authors should state which version of the asset is used and, if possible, include a URL.
        \item The name of the license (e.g., CC-BY 4.0) should be included for each asset.
        \item For scraped data from a particular source (e.g., website), the copyright and terms of service of that source should be provided.
        \item If assets are released, the license, copyright information, and terms of use in the package should be provided. For popular datasets, \url{paperswithcode.com/datasets} has curated licenses for some datasets. Their licensing guide can help determine the license of a dataset.
        \item For existing datasets that are re-packaged, both the original license and the license of the derived asset (if it has changed) should be provided.
        \item If this information is not available online, the authors are encouraged to reach out to the asset's creators.
    \end{itemize}

\item {\bf New assets}
    \item[] Question: Are new assets introduced in the paper well documented and is the documentation provided alongside the assets?
    \item[] Answer: \answerNA{} 
    \item[] Justification: 
    \item[] Guidelines:
    \begin{itemize}
        \item The answer \answerNA{} means that the paper does not release new assets.
        \item Researchers should communicate the details of the dataset\slash code\slash model as part of their submissions via structured templates. This includes details about training, license, limitations, etc. 
        \item The paper should discuss whether and how consent was obtained from people whose asset is used.
        \item At submission time, remember to anonymize your assets (if applicable). You can either create an anonymized URL or include an anonymized zip file.
    \end{itemize}

\item {\bf Crowdsourcing and research with human subjects}
    \item[] Question: For crowdsourcing experiments and research with human subjects, does the paper include the full text of instructions given to participants and screenshots, if applicable, as well as details about compensation (if any)? 
    \item[] Answer: \answerNA{} 
    \item[] Justification: 
    \item[] Guidelines:
    \begin{itemize}
        \item The answer \answerNA{} means that the paper does not involve crowdsourcing nor research with human subjects.
        \item Including this information in the supplemental material is fine, but if the main contribution of the paper involves human subjects, then as much detail as possible should be included in the main paper. 
        \item According to the NeurIPS Code of Ethics, workers involved in data collection, curation, or other labor should be paid at least the minimum wage in the country of the data collector. 
    \end{itemize}

\item {\bf Institutional review board (IRB) approvals or equivalent for research with human subjects}
    \item[] Question: Does the paper describe potential risks incurred by study participants, whether such risks were disclosed to the subjects, and whether Institutional Review Board (IRB) approvals (or an equivalent approval/review based on the requirements of your country or institution) were obtained?
    \item[] Answer: \answerNA{} 
    \item[] Justification: 
    \item[] Guidelines:
    \begin{itemize}
        \item The answer \answerNA{} means that the paper does not involve crowdsourcing nor research with human subjects.
        \item Depending on the country in which research is conducted, IRB approval (or equivalent) may be required for any human subjects research. If you obtained IRB approval, you should clearly state this in the paper. 
        \item We recognize that the procedures for this may vary significantly between institutions and locations, and we expect authors to adhere to the NeurIPS Code of Ethics and the guidelines for their institution. 
        \item For initial submissions, do not include any information that would break anonymity (if applicable), such as the institution conducting the review.
    \end{itemize}

\item {\bf Declaration of LLM usage}
    \item[] Question: Does the paper describe the usage of LLMs if it is an important, original, or non-standard component of the core methods in this research? Note that if the LLM is used only for writing, editing, or formatting purposes and does \emph{not} impact the core methodology, scientific rigor, or originality of the research, declaration is not required.
    \item[] Answer: \answerNA{} 
    \item[] Justification: 
    \item[] Guidelines:
    \begin{itemize}
        \item The answer \answerNA{} means that the core method development in this research does not involve LLMs as any important, original, or non-standard components.
        \item Please refer to our LLM policy in the NeurIPS handbook for what should or should not be described.
    \end{itemize}

\end{enumerate}

\end{document}